\documentclass{svproc}

\usepackage{bibunits}
\usepackage[round]{natbib}

\usepackage[utf8]{inputenc} 
\usepackage[T1]{fontenc}    
\usepackage{hyperref}       
\usepackage{url}            
\usepackage{booktabs}       
\usepackage{amsfonts}       
\usepackage{nicefrac}       
\usepackage{microtype}      
\usepackage{multirow}
\usepackage{placeins}
\usepackage{amsmath}
\usepackage{amssymb}
\usepackage{mathtools}
\usepackage{bm}
\usepackage{graphicx}
\usepackage{algorithm}
\usepackage{algpseudocode}
\usepackage{cleveref}
\usepackage{color}
\usepackage{comment}
\usepackage{url}

\usepackage{todonotes}
\usepackage{lipsum}
\usepackage{physics}
\usepackage{array}
\usepackage{enumitem}
\usepackage{wrapfig}
\usepackage{subcaption}

\usepackage[acronym]{glossaries}
\newacronym{MC}{MC}{Monte Carlo}
\newacronym{MCMC}{MCMC}{Markov chain Monte Carlo}
\newacronym{QMC}{QMC}{quasi-Monte Carlo}
\newacronym{CV}{CV}{\emph{control variate}}
\newacronym{CF}{CF}{\emph{control functional}}
\newacronym{SGD}{SGD}{stochastic gradient descent}

\DeclareMathOperator*{\argmin}{arg\,min}

\newenvironment{talign*}
 {\let\displaystyle\textstyle\csname align*\endcsname}
 {\endalign}
\newenvironment{talign}
{\let\displaystyle\textstyle\align}
{\endalign}

\newcommand{\papertitle}{Scalable Control Variates for Monte Carlo Methods via Stochastic Optimization}

\begin{document}

\mainmatter 

\title{Scalable Control Variates for Monte Carlo Methods via Stochastic Optimization}

\titlerunning{Scalable CVs for Monte Carlo Methods via Stochastic Optimization}  

\author{Shijing Si\inst{1} \and Chris. J. Oates\inst{2} \and Andrew B. Duncan\inst{3} \and Lawrence Carin\inst{1} \and Fran\c{c}ois-Xavier Briol\inst{4} }

\authorrunning{Si, Oates, Duncan, Carin, Briol} 

\institute{Duke University
\and  Newcastle University \and Imperial College London \and University College London } 

\maketitle        

\begin{abstract}
Control variates are a well-established tool to reduce the variance of Monte Carlo estimators.
However, for large-scale problems including high-dimensional and large-sample settings, their advantages can be outweighed by a substantial computational cost. 
This paper considers control variates based on Stein operators, presenting a framework that encompasses and generalizes existing approaches that use polynomials, kernels and neural networks. 
A learning strategy based on minimising a variational objective through stochastic optimization is proposed, leading to scalable and effective control variates. 
Novel theoretical results are presented to provide insight into the variance reduction that can be achieved, and an empirical assessment, including applications to Bayesian inference, is provided in support. 
\end{abstract}


\section{Introduction}
\label{sec:introduction}

This paper focuses on the approximation of the integral of an arbitrary function $f:\mathbb{R}^d \rightarrow \mathbb{R}$ with respect to a distribution $\Pi$, denoted $\Pi[f] := \int f \mathrm{d}\Pi$. It will be assumed that $\Pi$ admits a smooth and everywhere positive Lebesgue density $\pi$ such that the gradient of $\log \pi$ can be pointwise evaluated. 
This situation is typical in Bayesian statistics, where $\Pi$ represents a posterior distribution and, to circumvent this intractability, \gls{MCMC} methods are used. 
Nevertheless, the ergodic average of \gls{MCMC} output converges at a slow rate proportional to $n^{-1/2}$ and, for finite chain length $n$, there can be considerable stochasticity associated with the \gls{MCMC} output. 

A \gls{CV} is a variance reduction technique for \gls{MC} methods, including \gls{MCMC}. 
Given a test function $f$, the general approach is to identify another function, $g$, such that the variance of the estimator with $f$ replaced by $f - g$ is smaller than that of the original estimator, and such that $\Pi[g] = 0$, so the value of the integral is unchanged.
Such a $g$ is called a \gls{CV}. 
\gls{CV}s are widely-used in statistics and machine learning, including for the simulation of Markov processes \citep{Newton1994,Henderson2002}, stochastic optimization \citep{Wang2013}, stochastic gradient MCMC \citep{Baker2018}, reinforcement learning \citep{Greensmith2004,Grathwohl2017,Liu2018}, variational inference \citep{Paisley2012,Ranganath2014,Ranganath2016} and Bayesian evidence evaluation \citep{Oates2016thermo}.

Given a test function $f$, the problem of selecting an appropriate \gls{CV} is non-trivial and a variety of approaches have been proposed. 
Our discussion focuses only on the setting where $\pi$ is provided only up to an unknown normalization constant; i.e., the setting where \gls{MCMC} is typically used.
The most widely-used approach to selection of a \gls{CV} is based on $g = \nabla \log \pi$ and simple (e.g., linear) transformations thereof \citep{Assaraf1999,Mira2013,Friel2014,Papamarkou2014}; note that under weak tail conditions on $\pi$, the \gls{CV} property $\Pi[g] = 0$ is assured. Recently several authors have proposed the use of more complicated or even non-parametric transformations, such as based on high order polynomials \citep{South2019}, kernels \citep{Oates2017,Oates2016CF2,Barp2018} and neural networks (NNs) \citep{Grathwohl2017,Liu2018,Zhu2018}. These new approaches have been shown empirically -- and theoretically, in the case of kernels \citep{Barp2018,Oates2016CF2}-- to provide substantial reduction in variance for \gls{MCMC}. 

These recent developments are closely related to Stein's method \citep{Stein1972,Chen2011,Ross2011,Anastasiou2021}, a tool used in probability theory to quantify how well one distribution $\Pi'$ approximates another distribution $\Pi$.
Recall that, given a collection of functions $g$ for which $\Pi[g] = 0$ is satisfied, Stein's method uses $\sup_g \Pi'[g]$ as a means of quantifying the difference between $\Pi$ and $\Pi'$. 
As a byproduct, researchers in this field have constructed a large range of functions $g$ that can be used as \gls{CV}s. 
Although Stein's method has recently been applied to a variety of problems including \gls{MCMC} convergence assessment \citep{Gorham2015,Gorham2017,Gorham2016}, goodness-of-fit testing \citep{Chwialkowski2016,Liu2016,Yang2018}, variational inference \citep{Ranganath2014,Ranganath2016}, estimators for models with intractable likelihoods \citep{Barp2019,Liu2018Fisher} and the approximation of complex posterior distributions \citep{Liu2016,Liu2016SVGD,Liu2016BBIS,Chen2018,Chen2019,Riabiz2020}, a unified account of how Stein's method can be exploited for the construction of \gls{CV}s, encompassing existing polynomial, kernel and NN transformations, has yet to appear. 

The organization and contributions of this paper are as follows. The literature on polynomial, kernel, and NN \gls{CV}s is reviewed in \Cref{sec:background}. An efficient learning strategy for \gls{CV}s based on stochastic optimization is proposed in \Cref{sec:methodology}. A theoretical analysis is provided in \Cref{sec:theory}, which provides general sufficient conditions for variance reduction to be achieved. Finally, an empirical assessment is provided in \Cref{sec:experiments} and covers a range of synthetic test problems, as well as problems arising in the Bayesian inferential context.

\section{Background}
\label{sec:background}

In what follows, it is assumed that an approximate sample $\{x_i\}_{i=1}^n \subset \mathbb{R}^d$ from $\Pi$ have been obtained and our goal is to construct an estimator for $\Pi[f]$ of the form $\frac{1}{n-m} \sum_{i=m+1}^n f(x_i) - g(x_i)$ where $g$ is a CV learned using a subset of size $m \leq n$ from the $\{x_i\}_{i=1}^n$.

Several approaches have been proposed. 
One approach is to use a Taylor expansion of the test function $f$ \citep{Paisley2012,Wang2013}, or perhaps a polynomial approximation to $f$ learned from regression \citep{Leluc2019}. 
Unfortunately, this will only be a feasible approach when integrating against simple probability distributions $\Pi$ for which polynomials can be exactly integrated, such as a Gaussian. 
CVs may also be directly available through problem-specific knowledge \citep[e.g., for certain Markov processes;][]{Newton1994,Henderson2002}, but this is rarely the case in general. 
Alternatively, CVs can sometimes be built using known properties of the method used for obtaining samples; see \cite{Andradottir1993,Hammer2008,Dellaportas2012,Brosse2018,Belomestny2019,Belomestny2019martingale} for CVs that are developed with a particular \gls{MCMC} method in mind. 
See also \cite{Hickernell2005} for CVs specialized to \gls{QMC}. 
An obvious drawback to the methods above is that they impose strong restrictions on the methods that one may use to obtain the $\{x_i\}_{i=1}^m$. 

An arguably more general framework, and our focus in this paper, is to first curate a rich set $\mathcal{G}$ of candidate CVs, and then to employ a learning procedure to approximately select an optimal CV $g \in \mathcal{G}$. This should be done according to a suitable optimality criterion based on $f$ and the given set $\{x_i\}_{i=1}^m$. 
The methodological challenges are therefore twofold; first, we must construct $\mathcal{G}$ and second, we must provide a procedure to select a suitable CV from this set. The construction of a candidate set $\mathcal{G}$ has been approached by several authors using a variety of regression-based techniques:

\begin{itemize}[leftmargin=*]
    \item Motivated by physical considerations, \cite{Assaraf1999} proposed to use $g = H u$, based on the Schr\"{o}dinger-type Hamiltonian
    $
    H = - 0.5 \Delta + 0.5( \sqrt{\pi})^{-1} \Delta \sqrt{\pi} ,
    $
    where $\Delta$ is the Laplacian and $u$ is a polynomial of fixed degree.
    See also \cite{Mira2013,Friel2014,Papamarkou2014}.
    \item An approach called \emph{control functionals} (CFs) was proposed in \cite{Oates2017}, where the set $\mathcal{G}$ consisted of functions of the form 
    $
    g = \nabla \cdot u + u \cdot \nabla \log \pi ,
    $
    where $\nabla \cdot$ denotes the divergence operator, $\nabla$ denotes the gradient operator and $u : \mathbb{R}^d \rightarrow \mathbb{R}^d$ is constrained to belong to a suitable Hilbert space of vector fields on $\mathbb{R}^d$.
    See also \cite{Barp2018,Oates2016CF2,South2020} for the connection with Stein's method.
    \item In more recent work, \cite{Zhu2018} extended the CF approach to the case where a NN is used to provide a parametric family of candidates for the vector field $u$. 
    The set of all such functions $g$ generated using a fixed architecture of NN is taken as $\mathcal{G}$.
    See also \cite{tucker2017rebar,Liu2018}.
\end{itemize}
Thus, several related options are available for constructing a suitable candidate set $\mathcal{G}$.
However, where existing literature diverges markedly is in the procedure used to select a suitable CV from this set:
\begin{itemize}[leftmargin=*]
    \item For approaches based on polynomials, \cite{Assaraf1999} proposed to select polynomial coefficients $\theta$ in order to minimize the sum-of-squares error $\sum_{i=1}^m (f(x_i) - g_\theta(x_i))^2$.
    Here $g_\theta$ is used to emphasize the dependence on coefficients $\theta$ of the polynomial.
    For even moderate degree polynomials, the combinatorial explosion in the number of coefficients as $d$ grows necessitates regularized estimation of $\theta$; suitable regularizers are evaluated in \cite{Portier2018,South2019}.
    \item For the CF approaches, regularized estimation is essential since the Hilbert space is infinite dimensional.
    Here, \cite{Oates2017} proposed to select $g$ as a minimal norm element of the Hilbert space for which the interpolation equations $f(x_i) = c + g(x_i)$ are satisfied for all $i = 1,\dots,m$ and some $c \in \mathbb{R}$.
    A major drawback of this approach is the $O(m^3)$ computational cost.
    \item The approach based on NN also exploited a sum-of-squares error, but in \cite{Zhu2018} the authors proposed to include an additional regularizer term $\lambda \sum_{i=1}^m g_\theta(x_i)^2$, for some pre-specified constant $\lambda$, to avoid over-fitting of the NN.
    Optimization over $\theta$, the parameters of the NN that enter into $g_\theta$, was performed using stochastic gradient descent.
\end{itemize}
It is therefore apparent that, in existing literature, the construction of the candidate set $\mathcal{G}$ is intimately tied to the approach used to select a suitable element from it.
This makes it difficult to draw meaningful conclusions about which CVs are most suitable for a given task; from a theoretical perspective, existing analyses make assumptions that are mutually incompatible and, from a practical perspective, the different techniques and software involved in implementing existing methods precludes a straightforward empirical comparison.
Our attention therefore turns next to the construction of a general framework that can be used to learn a wide range of CVs, including polynomial, kernel and NN, under a single set of theoretical assumptions and algorithmic parameters, enabling a systematic assessment of CV methods to be performed.

\section{Methods}
\label{sec:methodology}

Here we present a general framework for the construction of CVs:
In \Cref{subsec: Stein sec} the construction of a candidate set $\mathcal{G}$ is achieved using Stein operators, which unifies the CVs proposed in existing contributions such as \cite{Assaraf1999,Oates2017,Zhu2018} and covers simultaneously the case of polynomials, kernels and NNs.
Then, in \Cref{subsec: SGD section}, we present an approach to selection of a suitable element $g \in \mathcal{G}$, based on a variational formulation and performing stochastic optimization on an appropriate objective functional.

\subsection{Classes of Control Variates $\mathcal{G}$} \label{subsec: Stein sec}

The construction of non-trivial functions $g : \mathbb{R}^d \rightarrow \mathbb{R}$ with the property $\Pi[g] = 0$ is not straight-forward in the setting where MCMC would be used, since for general $f$ the integral $\Pi[f]$ cannot be exactly computed.
Stein's method \citep{Stein1972} offers a solution to this problem in the case where the gradient of $\log \pi$ can be evaluated pointwise, which we describe next. A \emph{Stein characterization} of a distribution $\Pi$ consists of a pair $(\mathcal{U},\mathcal{L})$, where $\mathcal{U}$ is a set of functions whose domain is $\mathbb{R}^d$ and $\mathcal{L}$ is an operator, such that $\Pi'\left[\mathcal{L}u\right] = 0$ $ \forall u \in \mathcal{U}$ if and only if the distributions $\Pi'$ and $\Pi$ are equal. In this case $\mathcal{U}$ is called a \emph{Stein class} and $\mathcal{L}$ is called a \emph{Stein operator}\footnote{To simplify presentation in the paper, we always assume $\mathcal{U}$ is a \emph{maximal} set of functions for which $\mathcal{L}u$ is well-defined and $\Pi[\mathcal{L}u] = 0$.}. 
Clearly, if one can identify a Stein characterization for $\Pi$, then one could take $\mathcal{G} = \mathcal{L} \mathcal{U} = \{ \mathcal{L}u : u \in \mathcal{U} \}$ as a set of candidates CVs. 

The literature on Stein's method provides general approaches to identify a Stein characterization \citep{Chen2011,Ross2011}.
In the \emph{generator approach}, $\mathcal{L}$ is taken to be the infinitesimal generator of a Markov process which is ergodic with respect to $\Pi$ \citep{Barbour1988}. 
For example, if $\mathcal{L}$ is the infinitesimal generator of an overdamped Langevin diffusion then one obtains the \emph{Langevin} Stein operator, which acts on vector fields $u$ on $\mathbb{R}^d$ as $\mathcal{L}_{\text{L}} u  = \nabla \log \pi \cdot u + \nabla \cdot u$.
This recovers the operator used in the \gls{CF} approach of \cite{Oates2017}, as well as the operator used in the NN approach of \cite{Zhu2018}. 
Alternatively, we could construct an operator that acts on \emph{scalar}-valued functions by replacing the vector field $u$ with the potential $\nabla u$ in the previous operator, leading to the scalar-valued Langevin (SL) Stein operator $\mathcal{L}_{\text{SL}} u = \Delta u + \nabla u \cdot \nabla \log \pi$. This recovers the operator used with polynomials in \cite{Assaraf1999,Mira2013}.
Trivially, a scalar multiple of a Stein operator is a Stein operator, and one may combine Stein characterizations $(\mathcal{U}_i,\mathcal{L}_i)$ linearly as $\mathcal{L}u = \mathcal{L}_1 u_1 + \mathcal{L}_2 u_2$, $u \in \mathcal{U}_1 \times \mathcal{U}_2$, so that considerable flexibility can be achieved. We will see in Section \ref{sec:experiments} that this can lead to scalable and flexible classes of CVs.

\subsection{Selection of a Control Variate $g \in \mathcal{G}$} \label{subsec: SGD section}

Once a set $\mathcal{G}$ of candidate CVs has been constructed, we must consider how to select a suitable element $g \in \mathcal{G}$ (or equivalently $u \in \mathcal{U}$) that leads to improved performance of the \gls{MC} estimator when $f$ is replaced by $f - g$. 
In general this will depend on the specific details of the \gls{MC} method; for example, in MCMC one would select $g$ to minimize asymptotic variance \citep{Dellaportas2012,Belomestny2019}, while in \gls{QMC} one would minimize the Hardy-Krause variation \citep{Hickernell2005}.
The situation simplifies considerably when  $\mathcal{G}$ contains an element $g^*$ such that $f - g^*$ is constant.
This optimal function $g^* = \mathcal{L} u^*$, if it exists, is given by the solution of \emph{Stein's equation}: $\mathcal{L}u^*(x)  =  f(x) - \Pi[f]$. 
This paper proposes to directly approximate a solution of this equation (a linear partial differential equation) by casting it in a variational form and solving over a subset $\mathcal{V} \subseteq \mathcal{U}$. The variational characterization that we use is that $J(u^*) = 0$, where 
\begin{talign*}
\textstyle J(u) := \left\|  f - \mathcal{L}u -\Pi[f] \right\|_{L^2(\Pi)}^2 = \text{Var}_\Pi [f-\mathcal{L}u] ,
\end{talign*}
with $\textstyle L^2(\Pi)$ being the space of
square-integrable functions with respect to $\Pi$. 
In the spirit of empirical risk minimization, we propose to minimize an empirical approximation of this functional, computed based on samples $(x_i)_{i=1}^m$ that are drawn either exactly or approximately from $\Pi$.
There are two natural approximations that could be considered.
The first is based on the \emph{variance} representation
\begin{align}
J(u) & = \text{Var}_\Pi [f-\mathcal{L}u] \approx J_m^{\text{V}}(u) \label{eq:estimated_variance_loss} \\ 
J_m^{\text{V}}(u) & := \textstyle \frac{2}{m(m-1)} \sum_{i>j} (f(x_i)-\mathcal{L}u(x_i) - f(x_j)+\mathcal{L}u(x_j))^2 , \nonumber
\end{align}
providing an approximation of $J$ at cost $O(m^2)$, used in  \cite{Belomestny2017}.
The second is based on the \emph{least-squares} representation
\begin{align}
J(u) & = \min_{c \in \mathbb{R}} \|f-\mathcal{L}u - c \|_{L^2(\Pi)}^2 \approx \textstyle \min_{c \in \mathbb{R}} J_m^{\text{LS}}(c,u),  \label{eq:estimated_ls_loss} \\
J_m^{\text{LS}}(c,u) & := \textstyle \frac{1}{m} \sum_{i=1}^m \left(f(x_i)-\mathcal{L}u(x_i) - c \right)^2 , \nonumber
\end{align}
providing an approximation of $J$ at cost $O(m)$, used in \cite{Assaraf1999,Mira2013,Oates2017,Oates2016CF2}. These approximations will be unbiased when the $x_i$ are independent draws from $\Pi$, but this will not necessarily hold for the MCMC case. To approximately solve this variational formulation we consider a parametric subset $\mathcal{V} \subseteq \mathcal{U}$, where elements of $\mathcal{V}$ can be written as $v_\theta$ for some parameter $\theta \in \mathbb{R}^p$.
Depending on the specific nature of the functions $g_\theta$, it can occur that the optimization problem is under-constrained, e.g., when $p > m$.
Therefore, following \cite{Oates2017,South2019,Zhu2018}, we also allow for the possibility of additional regularization at the level of $\theta$.
Thus we aim to minimize objectives of the form $\tilde{J}_m^{\text{V}}(\theta) + \lambda_m \Omega(\theta)$ and $\tilde{J}_m^{\text{LS}}(c,\theta) + \lambda_m\Omega(\theta)$ over $c \in \mathbb{R}$ and $\theta \in \mathbb{R}^p$, where $\tilde{J}_m^{\text{V}}(\theta) := J_m^{\text{V}}(v_\theta)$, $\tilde{J}_m^{\text{V}}(c,\theta) := J_m^{\text{V}}(c,v_\theta)$, $\lambda_m > 0$ and $\Omega(\theta)$ is a regularization term to be specified.
To reduce notational overhead, for the least-squares case we let $\theta_0 := c$ and simply write $\tilde{J}_m^{\text{LS}}(\theta)$ where $\theta \in \mathbb{R}^{p+1}$.

To perform the minimization, we propose to use \gls{SGD}. 
Thus, to minimize a functional $F(\theta)$, we iterate through $\theta^{(t+1)} = \theta^{(t)} - \alpha_t \widehat{ \nabla F } (\theta^{(t)})$, where the \emph{learning rate} $\alpha_t$ decreases as $t \rightarrow \infty$ and $\widehat{\nabla F}$ is an unbiased approximation to $\nabla F$.
In our experiments, $\widehat{\nabla F}$ is constructed using a randomly chosen subset from $(x_i)_{i=1}^m$, with this subset being re-sampled at each step of \gls{SGD} (\emph{i.e.}, mini-batch \gls{SGD}) \citep{polyak1992acceleration,zhang2004solving}. 

This framework is compatible with any parametric function class and has the potential to provide significant speed-ups, relative to existing methods, due to the efficiency of \gls{SGD}. 
For example, taking $\mathcal{V}$ to be the polynomials of degree at most $k$ in each variable recovers the same class as \cite{Assaraf1999,Mira2013,Papamarkou2014,South2019}, but with a parameter optimization strategy based on \gls{SGD} as opposed to exact least squares. This problem is hence closely related to the ADALINE algorithm of \cite{Widrow1960} with basis functions which integrate to zero.

For \gls{SGD} with $t$ iterations and mini-batches of size $b$, our computational cost will be of order $\mathcal{O}(d^k b t)$, whereas exact least squares must solve a linear system of size $O(d^k)$, leading to a cost of $\mathcal{O}(d^{3k}+md^k)$. Similarly, taking $\mathcal{V}$ to be a linear space spanned by $m$ translates of a kernel recovers the \gls{CF} method of \cite{Oates2017}. 
\gls{SGD} has computational cost of $\mathcal{O}(m d b t)$, whereas \gls{CF}s requires $\mathcal{O}(m^3+m^2d)$ due to the need to invert an $m$-dimensional matrix. 
\begin{figure*}[t!]
\begin{center}
\includegraphics[width= \textwidth]{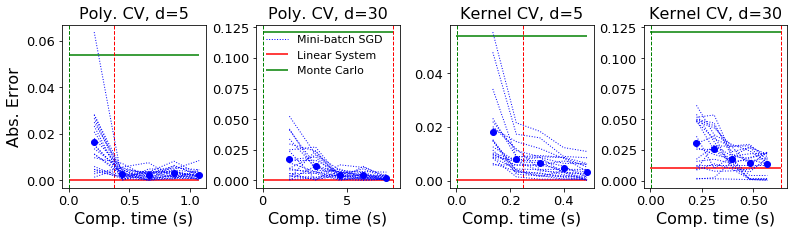}
\caption{\emph{Scalable Control Variates in High Dimensions.}
Here we consider the toy problem of integrating $f(x)=x_1+\dots+x_d$ against $\mathcal{N}(0,I_{d \times d})$. The total sample size is $n=1000$ and $m=500$ of these were used as the training set. 
Here $20$ realizations (blue dashed lines) are shown and blue dots represent the mean absolute error.
The red lines represent the performance and computational cost of solving the corresponding linear system exactly, our benchmark. 
Similarly, the green lines represent the \gls{MC} estimator with no \gls{CV} used.
}
\label{fig:polynomial_integrands}
\end{center}
\end{figure*}
Significant reduction in computational cost can also be obtained for ensembles: a combination of polynomial and kernel basis functions, as considered in \cite{South2020}, would cost $\mathcal{O}((md+d^k) bt)$ compared to the $\mathcal{O}(m^3+d^{3k}+m^2+md^k)$ cost when the linear system is exactly solved. Furthermore, any hyper-parameters, such as kernel parameters, can be incorporated into the minimization procedure with \gls{SGD}, so that nested computational loops are avoided.

Some of these speed-ups are illustrated on a toy example in  \Cref{fig:polynomial_integrands}. 
Even for this moderately-sized problem, the use of \gls{SGD} provides significant speed-ups. Additional experiments with values of $m=5000$ in Appendix \ref{appendix:polynomial_integrands} show that larger speed-ups can be obtained for large scale problems. 

\section{Theoretical Assessment}
\label{sec:theory}

In this section we present our novel theoretical results for CVs trained using SGD. All proofs are contained in Appendix \ref{appendix:proofs}.

The first question is whether it is possible to obtain \emph{zero-variance} CVs, i.e. can we find a $u \in \mathcal{U}$ such that $J(u)=\text{Var}_{\Pi}[f-\mathcal{L}u]=0$. The answer is ``yes'' under regularity conditions on $\Pi$ and $\mathcal{L}$, and whenever $\mathcal{V}$ is large enough. In particular, a fixed parametric class may not be large enough, but we can consider a nested sequence of sets $\mathcal{V}_1 \subseteq \mathcal{V}_2 \subseteq \ldots$ such that $\cup_{p \in \mathbb{N}} \mathcal{V}_p$ is dense in $\mathcal{U}$. For example, $\mathcal{V}_p$ could be polynomials of degree $p$, or NNs with $p$ hidden units. 
\begin{proposition} \label{thm:zero_variance}
Let $\mathcal{U}$ be a normed space and $\mathcal{L}:\mathcal{U} \rightarrow L^2(\Pi)$ be a bounded linear operator.
Consider a sequence of nested sets $\mathcal{V}_1 \subseteq \mathcal{V}_2 \subseteq \ldots$ such that $\cup_{p \in \mathbb{N}} \mathcal{V}_p$ is dense in $\mathcal{U}$.
If $\exists u \in \mathcal{U}$ that solves the Stein equation $\mathcal{L}u=f - \Pi[f]$, then $\lim_{p \rightarrow \infty} \inf_{v \in \mathcal{V}_p} J(v) = 0$.
\end{proposition}
Of course, the existence of a solution to the Stein equation needs to be verified. 
This point has not yet, to the best of our knowledge, been addressed in the literature on CVs. 
Our next result below provides regularity conditions for the existence of a solution when using $\mathcal{L}_{\text{SL}}$, the Stein operator used in our experiments. Denote the Sobolev space $W^{k,p}(\Pi)$ of functions whose weak derivatives of order $k$ are in $L^p(\Pi)$ and the Sobolev space $W^{k,p}_{\text{loc}}$ of functions whose $p$-th power weak derivatives of order $k$ are locally integrable; these are formally defined in Appendix \ref{appendix:mathematical_background}. 
For a vector-valued function $h : \mathbb{R}^d \rightarrow \mathbb{R}^p$ we let $\textstyle \|h\|_{L^p(\Pi)} := ( \sum_{i=1}^d \|h_i\|_{L^p(\Pi)}^2 )^{1/2}$.
\begin{proposition} \label{thm:zero_variance_poisson} 
Consider the vector space $\mathcal{U} = W^{2,2}(\Pi)\cap W^{1,4}(\Pi)$ equipped with norm $\|u\|_{\mathcal{U}} := \max(\| u \|_{W^{1,4}(\Pi)},\| u \|_{W^{2,2}(\Pi)})$.  Then $\mathcal{L}_{\textsc{SL}}:\mathcal{U} \rightarrow L^2(\Pi)$ is a bounded linear operator with $\|\mathcal{L}_{\textsc{SL}}\|_{\mathcal{U} \rightarrow L^2(\Pi)} \leq 2 ( \|\nabla \log \pi\|^2_{L^4(\Pi)}+1 )^{\frac{1}{2}}$. 

Furthermore, suppose that
\begin{enumerate}
    \item[(i)] $\int \|x\|_2^{K} \mathrm{d}\Pi(x) < \infty$ for some $K > 8$,
    \item[(ii)] $(\nabla \log \pi)(x) \cdot ( x/\|x\|_2 ) \leq - r \|x\|_2^\alpha$ for some $\alpha > -1$, $r>0$, and all $\lVert x\rVert_2>M$ for some $M>0$,
    \item[(iii)] $|f(x)| \leq C_1 + C_2 \|x\|_2^\beta$ for some $C_1,C_2 \geq 1$ and $\beta < K/4 - 2$.
\end{enumerate}
Then, $\exists u \in \mathcal{U}$ that solves the Stein equation $\mathcal{L}_{\textsc{SL}} u = f - \Pi[f]$.
\end{proposition}
The fact that the space $\mathcal{U}$ in \Cref{thm:zero_variance_poisson} is separable ensures that suitable approximating sets $\mathcal{V}_p$ can be constructed.
For example, if $\{u_i\}_{i=1}^\infty$ is a spanning set for $\mathcal{U}$ then we may set $\mathcal{V}_p = \text{span}(u_1,\dots,u_p)$, in which case $\cup_{p \in \mathbb{N}} \mathcal{V}_p$ is dense in $\mathcal{U}$ so the result of  \Cref{thm:zero_variance_poisson} holds.

Notice that a solution to the Stein equation will not be unique, since one can introduce an additive constant.
This motivates, in practice, the use of an additional regularizer $\Omega(\theta)$ to ensure uniqueness of the minimum of $\theta \mapsto J(v_\theta)$. 

In \Cref{app: convex opt} of the Electronic Supplement we also recall a standard convergence result for SGD in settings where the objective is convex, focusing on the case where $\mathcal{G}$ is a finite dimensional linear space. This result is thus applicable to polynomials and kernels, but not NN-based CVs.

\section{Empirical Assessment}
\label{sec:experiments}

Here we assess our method on both synthetic problems and on problems arising in a Bayesian statistical context.
Our aim is twofold; (i) to assess whether our learning procedure provides a speed-up compared to existing approaches, and (ii) to gain insight into which class of CV may be most appropriate for a given context.
The Stein operator $\mathcal{L}_{\text{SL}}$ was used for all experiments. 
For the polynomial and kernel CVs, the regularizer $\Omega(\theta) = \|\theta\|_2^2$ was used, while for NN CVs the regularizer $\Omega(\theta) = \sum_{i=1}^{m}g_{\theta}(x_i)^2$ was used, following \cite{Zhu2018}. 
The regularization strength
parameter $\lambda$ was tuned by cross-validation.
For some datasets, we employed two ensemble CVs: a sum of kernel and a polynomial (i.e., kernel + polynomial); and a sum including two kernels with different hyperparameters and a polynomial (i.e., multiple kernels + polynomial). 
Implementation details and further experiments are provided in \Cref{appendix:additional_experiments} of the Electronic Supplement.

\paragraph{Genz Test Functions:}

The Genz functions are a standard benchmark used to evaluate a numerical integration method \citep{Genz1984}. 
These functions $f$ exhibit discontinuities and sharp peaks, but nevertheless they can be exactly integrated. 
The purpose of this first experiment is simply to assess whether \emph{any} variance reduction can be achieved using our general framework in challenging and pathological situations.\footnote{We emphasize that MC can be evaluated at negligible cost and we are not advocating that our methods should be preferred for this task.}
Results are shown in Table \ref{table:Genz_functions_compare_cv} for polynomial-based and kernel-based CVs, as well as an ensemble of both. The CVs are trained using SGD on the least-squares objective functional with batch size $b=8$ for 25 epochs. 
For each $f$, the mean absolute error (MAE) of polynomial CVs is always the largest while the linear combination of kernel and polynomial consistently performs the best. This is likely due to the increased flexibility of the CV. 
In all cases a substantial reduction in MAE was achieved, compared to MC.
Full details and an extensive range of additional experiments are provided in \Cref{appendix:Genz} of the Electronic Supplement.

\begin{table*}[t]
\begin{center} \small
\begin{tabular}{|c|cccc|}
\hline
  \textbf{Integrand } $f$ & \textbf{MC} & \textbf{Poly. CV}  & \textbf{Ker. CV}  & \textbf{Poly.+Ker. CV}  \\ 
  \hline \hline
  Continuous & 2.77e-03& 3.21e-03 & $3.28$e-$04$&  $\bm{1.85}$e-$\bm{04}$ \\
  Corner Peak & 5.76e-03 & 1.07e-03& ${9.27}$e-$06$& $\bm{6.05}$e-$\bm{06}$ \\
  Discontinuous&  2.04e-02& 1.32e-02& ${3.91}$e-${03}$& $\bm{2.65}$e-$\bm{03}$ \\
  Gaussian Peak&   1.47e-03& 1.40e-03& ${1.24}$e-${05}$& $\bm{1.05}$e-$\bm{05}$ \\
  Oscillatory&  4.17e-03& 1.06e-03& $4.63$e-${06}$& $\bm{3.90}$e-$\bm{06}$ \\
  Product Peak&  1.37e-03& 1.32e-03& ${2.12}$e-${05}$& $\bm{2.52}$e-$\bm{06}$ \\
  \hline \hline
  Time (sec.) & 7.10e-02 & 4.30e+00 & 2.60e+00 & 5.70e+00  \\
  \hline
 \end{tabular}
 \vspace{2mm}
 \caption{\emph{Mean absolute error (based on 20 repetitions) for polynomial-based CV, kernel-based CV 
 and an ensemble of these, for the Genz benchmark \citep{Genz1984}}. We took $n=1000, m=500$ and $d=1$. The training time presented is for 25 epochs, averaged over repetitions for all integrands.
 }
\label{table:Genz_functions_compare_cv}
\end{center}
\end{table*}

\paragraph{Integrating Gaussian Processes:}
To automatically generate test problems, we modelled $f$ as a Gaussian process (GP) and sampled $(\Pi[f],f(x_1), \dots, f(x_n))$ from its Gaussian marginal; here the GP was centred and a squared-exponential covariance function was used, and the distribution $\Pi$ was taken to be an $L$-component Gaussian mixture model. 
In this way infinitely many problem instances can be generated, of a similar nature to those arising in computer experiments \citep{Kennedy2001} and Bayesian numerical methods \citep{OHagan1991,Briol2019PI}. 
We compared CVs based on polynomials, 
kernels, and NNs (three-layer ResNet with ReLU activation with $50$ neurons per layer). 

\begin{figure*}[t]
\begin{center}
\includegraphics[width=\textwidth]{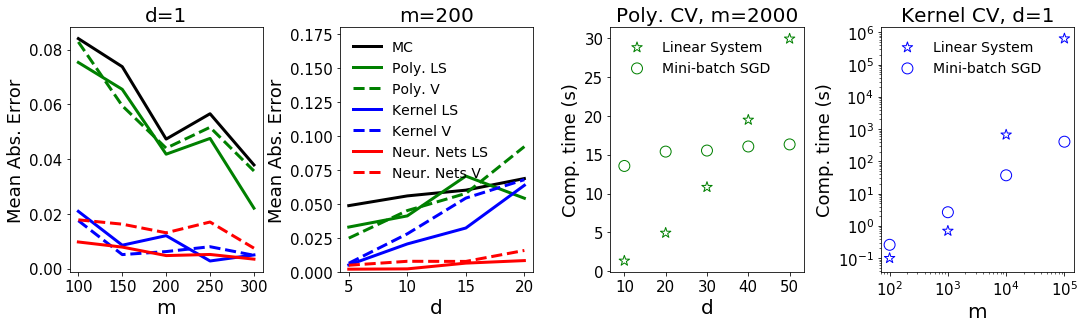}
\caption{\emph{Integrating Gaussian Processes.}
\textit{Left and centre-left:} The mean absolute error (based on 20 repetitions) of the CV estimators as a function of the training set size $m$ and dimension $d$.
\textit{Centre-right and right:} Compute times for polynomial and kernel CVs as a function of $m$ and $d$.
}
\label{fig:GPrealisations_experiments}
\end{center}
\end{figure*}

Results are presented in Fig. \ref{fig:GPrealisations_experiments}, with implementational details in \Cref{appendix:GP_realisations} of the Electronic Supplement.  
The left-most panel presents the performance of each CV for minimising either $\tilde{J}^{\text{V}}_m$ or $\tilde{J}^{\text{LS}}_m$ in $d=1$. 
Polynomials are not flexible enough for such complex integrands, but kernels and NNs can achieve substantial reduction in error.
However, we found that the ``effective'' time requires to implement a NN, including initialization of SGD and selecting an appropriate learning rate, meant that NN were not time-competitive with the other methods considered.
The center-left panel studies the impact of $d$ on the performance of each method.  The performance of polynomial and kernel CVs degrades rapidly with $d$, but this is not the case for NNs. 
In both panels, $\tilde{J}_m^{\text{LS}}$ leads to improved results compared to $\tilde{J}_m^{\text{V}}$.
The centre-right and right panels report computational times of
linear system and mini-batch SGD
as $d$ and $m$ grows. 
These two panels verify that mini-batch SGD has linear time complexity as $n$ or $d$ is increased, whist exact solution of linear systems leads to exponential computational costs for polynomial and kernel CVs.

\paragraph{Parameter Inference for Ordinary Differential Equations:}
Here we consider the problem of inference for parameters $\alpha,\beta,\gamma,\delta$ of the Lotka–Volterra equations $\dot{x}  = \alpha x -\beta xy$,  $\dot{y}  = \delta xy - \gamma y$, a popular ecological model for competing populations \citep{lotka1925principles,volterra1926fluctuations}.
Our experimental set up is identical to that used in \cite{Riabiz2020}.
Our task is to compute posterior means of these dynamic parameters based on datasets of size $n$ arising as a subsample from Metropolis-adjusted Langevin algorithm output \citep{roberts1996}; the full MCMC output provided the ground truth.
Half of the sample was used to
train CVs ($m=\frac{n}{2}$) and a batch size of $b=8$ was used over $25$ epochs in SGD based on the least-squares objective functional.

\begin{figure*}[t!]
\begin{center}
\includegraphics[width=\textwidth]{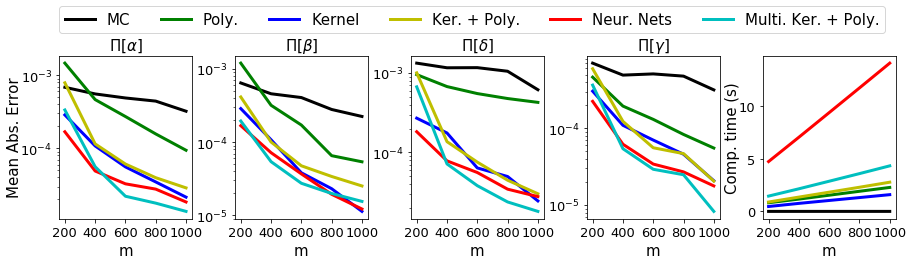}
\caption{\emph{Parameter Inference for Ordinary Differential Equations.} 
Each panel except the rightmost presents the mean absolute error (based on 20 repetitions) for approximation of posterior expectations of model parameters using MCMC output. The rightmost panel presents the computing time of training these CVs.
Here ``MC'' represents the benchmark where no CV is used.
}
\label{fig:pde_results}
\end{center}
\end{figure*}

Fig. \ref{fig:pde_results} displays the performance of different CVs under sizes of training dataset. 
In each case the standard MC estimate is outperformed, with ensemble of multiple kernels with a polynomial or the NN performing uniformly best. Due to the computational cost of training NNs as shown in the rightmost panel, we found the ensemble to be preferable. The ensemble also leads to a convex objective which is easier to minimize.

\paragraph{High-dimensional Bayesian Logistic Regression}
In this final example, we consider Bayesian logistic regression. We experimented on two different datasets: the Sonar data and the Madelon data. The Sonar dataset has dimension $d=61$, which is lower than the $d=500$ of the Madelon dataset. Results were similar for both experiments, and the Sonar data is therefore relegated to \Cref{appendix:sonar_data} of the Electronic Supplement. 

The Madelon data is an artificial dataset, which was part of the NIPS/NeurIPS 2003 feature selection challenge \citep{Guyon2003,dheeru2017uci}. This is a two-class classification problem with 500 continuous input variables. We denote by $\beta$ the weight vector that includes all parameters to infer in the Bayesian logistic regression.
MCMC was used to sample from the posterior of $\beta$ with the Python interface to Stan \citep{Carpenter2016}.
Our task is to approximate the posterior probability that an unlabeled data point $z$ corresponds to label 1, rather than 0, based on a subset of size $m$ from the MCMC output. Thus $f(\beta)=(1+\exp(-z^\top \beta))^{-1}$. The entire chain was used to establish ``ground truth'' for the value of this integral.

\begin{figure}[t!]
\centering
\includegraphics[width=.6\textwidth,trim={1mm 0 0mm 0},clip]{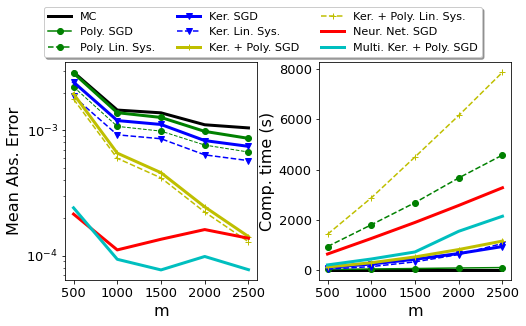}
\caption{\emph{Madelon Dataset}. 
The mean absolute error (left) and compute times (right), as a function of the size $m$ of the training set; based on 20 repetitions. 
}
\label{fig:madelon_logistic_results}
\end{figure}

In these experiments, $J^{\text{LS}}_m$ was used with $m=n$ and batch sizes of $b = 8$ over 25 epochs of SGD. 
Fig. \ref{fig:madelon_logistic_results} compares the performance of different CV methods. The two ensemble CVs and the NNs perform significantly better than other CVs. 
When $m < 1000$, the NNs and the CV with multiple kernels and a polynomial have similar performance, better than others. When $m \geq 1000$, the ensemble
CV surpasses NNs. One possible explanation is that for all values of $m$ we used the same multi-layer perceptron (MLP) with
6 layers and 20 nodes in each of them.
Therefore, the NNs size (capacity) remains the same while the training data size $m$ increases. Further growing the depth of NN could lead to an improved performance.
Furthermore, the results for polynomials and kernels demonstrate that our general framework based on SGD can achieve comparable MAE with exactly solving the linear systems, but with a fraction of the associated computational overhead.
The compute time of NN in Fig. \ref{fig:madelon_logistic_results} does not capture the time required to manually calibrate SGD, so that the ``effective'' compute time is much higher than reported.

\section{Conclusion}
\label{sec:conclusion}

This paper outlined a general framework for developing CVs using Stein operators and SGD. 
It was demonstrated that (i) the proposed training scheme leads to speed-ups compared to existing CV methods; (ii) novel CV methods (e.g., ensemble methods) can be easily developed; (iii) theoretical analysis can be performed in quite a general setting that simultaneously encompasses multiple CV methods. 
Further research could explore the use of other Stein classes and operators. 
In terms of Stein classes, one could consider the use of wavelets, which are known for their good performance for multi-scale function approximation, or other NN architectures which could provide further gains in high dimensions. 
Stein operators are not unique and one could explore parameterized operators \citep{Ley2016Parametric} and include these parameters in the optimization scheme. 
Finally, one could construct novel CVs on other spaces, such as general smooth manifolds or countable spaces \citep{Barp2018}. 
\subsubsection*{Acknowledgements}
The authors would like to thank Charline Le Lan for helpful discussions, and Wilson Chen, Marina Riabiz and Leah South for sharing MCMC samples from the model of atmospheric pollutants, the predator-prey model and the logistic regression model respectively. 
CJO, ABD, FXB were also supported by the Lloyd's Register Foundation Programme on Data-Centric Engineering and the Alan Turing Institute under the EPSRC grant [EP/N510129/1]. FXB was supported by an Amazon Research Award on ``Transfer Learning for Numerical Integration in Expensive Machine Learning Systems''.

\appendix

\section{Proofs of Theoretical Results} \label{appendix:proofs}

\subsection{Some Elements from Functional Analysis} \label{appendix:mathematical_background}

Let $X$ and $Y$ be two normed real vector spaces.
A function $f:X \rightarrow Y$ is called \emph{Lipschitz} continuous if there exists a constant $L$ such that, $\forall x,x' \in X$: $\|f(x)-f(x')\|_Y \leq L \|x-x'\|_X$. The smallest such $L \geq 0$ is called the \emph{Lipschitz constant} of $f$.
The norm of a bounded linear operator $\mathcal{L}:X \rightarrow Y$ is given by: $
\|\mathcal{L}\|_{X \rightarrow Y} := \inf \left\{ c \geq 0 : \|\mathcal{L} x\| \leq c \|y\| \; \forall x \in X  \right\}$.
For $1\leq p < \infty$ we denote 
\begin{talign*}
{L^p(\Pi)} & := \left\{ f:\mathbb{R}^d \rightarrow \mathbb{R} \text{ measurable } \Big| \|f\|_{L^p(\Pi)} := \left(\int_{\mathbb{R}^d} |f(x)|^p \Pi(\mathrm{d}x)\right)^{\frac{1}{p}} < \infty \right\}.\\
{L^p_{\text{loc}}} & := \left\{ f:\mathbb{R}^d \rightarrow \mathbb{R} \text{ measurable } \Big| \left(\int_{K} |f(x)|^p \mathrm{d}x\right)^{\frac{1}{p}} < \infty,  \; \forall \text{compact } K \subset \mathbb{R}^d \right\}.
\end{talign*}
As usual, $L^p(\Pi)$ can be interpreted as a normed space via identification of functions that agree $\Pi$-almost everywhere on $\mathbb{R}^d$.
Using this definition, we can now also define weighted Sobolev spaces of integer smoothness:
\begin{talign*}
W^{k,p}(\Pi) & := \left\{ f \in L^p(\Pi) \Big| D^\alpha f \in L^p(\Pi) \; \forall |\alpha| \leq k \right\}\\
W_{\text{loc}}^{k,p} & := \left\{ f \in L_{\text{loc}}^p \Big| D^\alpha f \in L_{\text{loc}}^p \; \forall |\alpha| \leq k \right\}
\end{talign*}
In this definition, $\alpha=(\alpha_1,\ldots,\alpha_d) \in \mathbb{N}^d_0$ is a multi-index and $D^\alpha$ denotes the weak derivative of order $\alpha$, \emph{i.e.} $D^\alpha f := \partial^{|\alpha|}f / \partial x_1^{\alpha_1} \ldots \partial x_d^{\alpha_d}$.
Recall that $W^{k,p}(\Pi)$ can be interpreted as a normed space with norm $
\| u \|_{W^{k,p}(\Pi)} := (\sum_{i=0}^k \sum_{\alpha: |\alpha|=i} \int | D^{\alpha} u(x) |^p \mathrm{d}\Pi(x))^{\frac{1}{p}}$,
again via identification of functions whose derivatives up to order $|\alpha| \leq k$ agree $\Pi$-almost everywhere on $\mathbb{R}^d$.

\subsection{Proof of  \Cref{thm:zero_variance}}

\begin{proof}
Let $u \in \mathcal{U}$ solve the Stein equation $\mathcal{L}u = f - \Pi[f]$.
Since $\mathcal{L}$ is a bounded linear operator between normed spaces,
\begin{talign*}
 J(v) 
=  \left\|f - \Pi[f] - \mathcal{L}v\right\|^2_{L^2(\Pi)} 
=  \left\|\mathcal{L}u - \mathcal{L}v\right\|^2_{L^2(\Pi)} 
\leq \|\mathcal{L}\|_{\mathcal{U} \rightarrow L^2(\Pi)}^2 \|u-v\|^2_{\mathcal{U}}
\end{talign*}
where $\|\mathcal{L}\|_{\mathcal{U}\rightarrow\mathcal{L}^2(\Pi)} < \infty$.
Fix $\epsilon > 0$.
Since $u \in \mathcal{U}$ and $\cup_{p \in \mathbb{N}} \mathcal{V}_p$ is dense in $\mathcal{U}$, there exists $v \in \cup_{p \in \mathbb{N}} \mathcal{V}_p$ such that $\|u - v \|_{\mathcal{U}} < \epsilon$.
In particular, there exists $q \in \mathbb{N}$ such that $v \in \mathcal{V}_q$.
Moreover, since $\mathcal{V}_q \subseteq \mathcal{V}_p$ for all $q \leq p$, the function $p \mapsto \inf_{v \in \mathcal{V}_p} J(v)$ is non-increasing.
Thus
\begin{align*}
0 \leq \lim_{p \rightarrow \infty} \inf_{v \in \mathcal{V}_p} J(v) \leq \inf_{v \in \mathcal{V}_q} J(v) & \leq \|\mathcal{L}\|_{\mathcal{U} \rightarrow L^2(\Pi)}^2 \inf_{v \in \mathcal{V}_q} \|u-v\|_{\mathcal{U}}^2 \leq \|\mathcal{L}\|_{\mathcal{U} \rightarrow L^2(\Pi)}^2 \epsilon^2 .
\end{align*}
Since $\epsilon > 0$ was arbitrary, the right hand side can be made arbitrarily small.
\end{proof}

\subsection{Proof of  \Cref{thm:zero_variance_poisson}}

\begin{proof}
First we will show that $\mathcal{L}_{\textsc{SL}}$ is a bounded linear operator from $\mathcal{U} = W^{2,2}(\Pi)\cap W^{1,4}(\Pi)$ to $L^2(\Pi)$.
To this end:
\begin{talign}
\left\|\mathcal{L}_{\textsc{SL}}u- \mathcal{L}_{\textsc{SL}}v\right\|^2_{L^2(\Pi)}
& = \left\|\nabla \log \pi \cdot \nabla (u- v) + \nabla \cdot \nabla (u- v)\right\|^2_{L^2(\Pi)} \label{eq:poissonproof_0} \\
&
\hspace{-70pt} \leq 2  \left[ \left\|\nabla \log \pi \cdot \nabla (u- v)\right\|^2_{L^2(\Pi)} + \left\|\nabla \cdot \nabla (u- v)\right\|^2_{L^2(\Pi)}\right] \label{eq:poissonproof_1} \\
& 
\hspace{-70pt} \leq 2 \left[ \left\| \nabla \log \pi \right\|^2_{L^4(\Pi)}\left\|\nabla( u- v ) \right\|^2_{L^4(\Pi)} + \left\|u- v\right\|^2_{W^{2,2}(\Pi)}\right] \label{eq:poissonproof_2} \\
&
\hspace{-70pt} \leq 2 \left(\left\| \nabla \log \pi \right\|^2_{L^4(\Pi)} + 1 \right) \left( \|u-v\|^2_{W^{1,4}(\Pi)} + \|u-v\|^2_{W^{2,2}(\Pi)} \right), \label{eq:poissonproof_3} \\
& 
\hspace{-70pt} \leq  4 \left(\left\| \nabla \log \pi \right\|^2_{L^4(\Pi)} + 1 \right) \max\left( \|u-v\|_{W^{1,4}(\Pi)} , \|u-v\|_{W^{2,2}(\Pi)} \right)^2 \nonumber
\end{talign}
\Cref{eq:poissonproof_0} follows by definition of the Stein operator, Eq. \ref{eq:poissonproof_1} follows from the fact that $(a+b)^2 \leq 2(a^2+b^2)$. 
\Cref{eq:poissonproof_2} follows from the vector-valued H\"older inequality together with the definition of $\|\cdot\|_{W^{2,2}(\Pi)}$.
\Cref{eq:poissonproof_3} follows from the definition of $\|\cdot\|_{W^{1,4}(\Pi)}$.
Thus $\mathcal{L}_{\textsc{SL}}$ is a bounded linear operator as claimed, and moreover $\|\mathcal{L}_{\text{SL}}\|_{\mathcal{U} \rightarrow L^2(\Pi)} \leq 2 (\|\nabla \log \pi \|_{L^4(\Pi)}^2 + 1)^{\frac{1}{2}}$.

The second task is to establish that there exists a solution to the Stein equation $\mathcal{L}_{\text{SL}} u = f - \Pi[f]$. For this we leverage \citet[Theorem 1]{Pardoux2001} which states that, if conditions (ii), (iii) hold, there exists a solution $u$ to the Stein equation which is continuous and belongs to $W_{\text{loc}}^{2,q}$ for all $q> 1$.  Moreover, $\forall m > \beta + 2$ there exists $C_m$ such that $|u(x)| + |\nabla u(x)| \leq C_m(1 + |x|^m)$ for all $x \in \mathbb{R}^d$. By assumption (i) it follows that $u \in W^{1,4}(\Pi)$.  Moreover, since $\pi$ was assumed to be smooth (recall, this was assumed at the outset in \Cref{sec:introduction}), standard regularity results imply that $u$ is smooth and so, is a classical solution.  We can therefore write
\begin{talign*}
|\Delta u(x) | \leq |f(x)| + |\Pi(f)| + |\nabla \log \pi(x)\cdot \nabla u(x)|, \quad x \in \mathbb{R}^d,
\end{talign*}
so that $\left\| \Delta u \right\|_{L^2(\Pi)} \leq 2\lVert f \rVert_{L^2(\Pi)} + \lVert\nabla \log \pi\rVert_{L^4(\Pi)}\lVert u \rVert_{W^{1,4}(\Pi)} < \infty$.  It follows that $u \in W^{2,2}(\Pi)\cap W^{1,4}(\Pi)$, as claimed.

\end{proof}

\bibliography{reference,bibliography}
\newpage

\begin{center}
\vspace{0.3cm}
\LARGE \textbf{Electronic Supplement} 
\end{center}

The following document supplements the paper \emph{\papertitle}. 
In \Cref{appendix:existing_methodology}, we review existing methodology for \gls{CV}s based on polynomials and kernels. 
\Cref{app: convex opt} discusses stochastic convex optimization in from a theoretical standpoint.
Finally, \Cref{appendix:additional_experiments} contains a detailed exposition of the experimental setup in the paper for reproducibility, and provides additional results. 

\section{Additional Background on Control Variates}\label{appendix:existing_methodology}

Let $\{x_i\}_{i=1}^n$ be a set containing approximate samples from $\Pi$.
The classic approach to \gls{CV}s is based on data-splitting, such that a \gls{CV} $g$ is constructed based on a subset of the samples $\{x_i\}_{i=1}^m$, then a \gls{MC} estimator based on $f-g$ is evaluated using the remainder of the samples, $\{x_i\}_{i=m+1}^n$.
Thus $\Pi[f]$ is approximated using the \gls{CV} estimator
\begin{talign*}
\frac{1}{(n-m)} \sum_{i=m+1}^n (f(x_i)-g(x_i))
\end{talign*}
where $g(\cdot) = g(\cdot; x_1,\dots,x_m)$.
If the $x_i$ are independent samples from $\Pi$ then such a \gls{CV} estimator is unbiased.
It is also common practice to use the same set $\{x_i\}_{i=1}^n$ for both the construction of $g$ and evaluation of the \gls{MC} estimator; in this case the estimator is biased in general but may enjoy superior mean square error.

In this section we recall existing approaches to constructing \gls{CV}s, providing references to existing literature where appropriate.

\subsection{Control Variates based on Polynomials}\label{appendix:poly}

As pointed out in the main text, the polynomial \gls{CV}s of \cite{Assaraf1999,Mira2013,Papamarkou2014,South2019} are based on $\mathcal{L}_{\text{SL}}$ and take the form:
\begin{talign*}
g_\theta(x) = \mathcal{L}_{\text{SL}} v_\theta(x) = \Delta_x v_\theta(x) + \nabla_x v_\theta(x) \cdot \nabla_{x}\log \pi(x),
\end{talign*}
where $v_\theta(x)$ is a polynomial of order $k\in \mathbb{N}$. For first order polynomials (i.e. $k=1$ and $p=d$), we have $v_\theta(x) = \sum_{i=1}^d \theta_i x_i$ where $\theta = (\theta_1,\ldots,\theta_d) \in \mathbb{R}^d$, and the \gls{CV} estimator is of the form: $g_\theta(x) =  \theta \cdot \nabla_x \log \pi(x)$.
Note that the constant term is not included, since this is in the null space of $\mathcal{L}_{\text{SL}}$.
More generally, for an arbitrary polynomial of order $k$,
 \begin{talign*}
 v_\theta(x) = \sum_{j=1}^p \theta_j x^{\alpha_{j1}}_1 \cdots x^{\alpha_{jd}}_d,
 \end{talign*}
 for some $\theta = (\theta_1,\ldots,\theta_p) \in \mathbb{R}^p$ and where the rows of the matrix $\alpha \in \mathbb{Z}^{p \times d}$ are multi-indices containing polynomial coefficients such that the polynomial has total degree $k\geq 1$: $1 \leq \sum_{l=1}^d \alpha_{jl} \leq p$. The total number of polynomials satisfying this condition is $p = {{d+k}\choose{d}}-1$. 
The \gls{CV}s based on such polynomials take the form $g_{\theta}(x) = \theta \cdot b(x)$, where the vector $b(x) = (b_1(x),\ldots,b_p(x))$ has components:
\begin{talign*}
b_j(x) & = \left[ \sum_{l=1}^d \max\left(0,\alpha_{jl}\right) x_l^{\alpha_{jl}-1} \frac{\partial \log \pi}{\partial x_l} \prod_{z=1,z \neq l}^d x_z^{\alpha_{jz}}\right] \\
& \qquad +  \left[\max\left(0,\alpha_{jl}(\alpha_{jl}-1)\right) x_l^{\alpha_{jl}-2} \prod_{z=1,z\neq l}^d x_z^{\alpha_{jz}}\right]
\end{talign*}
for $j=1,\ldots,p$; see Appendix A of \cite{South2019}. 
The value of $\theta$ which minimizes the least-squares objective $\hat{J}^{\text{LS}}_m$ is given by $\theta^*_m = \hat{V}_m^{-1} \hat{C}_m$ with: 
\begin{talign*}
\hat{V}_m & = \frac{1}{(m-1)} \sum_{i=1}^m \left(b(x_i) -  \frac{1}{m}\sum_{i=1}^m b(x_i) \right) \left(b(x_i) -  \frac{1}{m}\sum_{i=1}^m b(x_i)\right)^\top,  \\
\hat{C}_m & =  \frac{1}{(m-1)} \sum_{i=1}^m \left(f(x_i) - \frac{1}{m}\sum_{i=1}^m f(x_i)\right) \left(b(x_i) - \frac{1}{m}\sum_{i=1}^m b(x_i) \right)^\top .
\end{talign*}
The size $m$ of the training dataset is required to be sufficiently large to ensure that the matrix $\hat{V}_m$ is non-singular, otherwise additional regularisation is required \citep{South2019}.
Exact solution of this linear system for $\theta_m^*$ requires a computational cost of $O(p^3)$, which can be prohibitive since $p$ increases rapidly with both $d$ and $k$.

\subsection{Control Functionals: Control Variates based on Reproducing Kernels}
\label{appendix:kernel}

\gls{CF}s are \gls{CV}s constructed using a nonparametric kernel-based interpolant. Let $k:\mathbb{R}^d\times \mathbb{R}^d \rightarrow \mathbb{R}$ be a symmetric positive definite kernel with corresponding reproducing kernel Hilbert space $\mathcal{H}_k$.  \cite{Oates2017} noted that, under some regularity conditions, the kernels 
\begin{talign}
k_0(x,y)  & :=   \nabla_x \cdot \nabla_y k(x,y) + \nabla_x k(x,y) \cdot \nabla_y \log \pi(y) \nonumber \\ 
& \qquad + \nabla_y k(x,y) \cdot \nabla_x \log \pi(x) + k(x,y) \nabla_x \log \pi(x) \cdot \nabla_y \log \pi(y), \label{eq:kernel0}
\end{talign}
and $k_{+}(x,y) := k_0(x,y)+\sigma^2$ for $\sigma>0$ are also reproducing kernels with corresponding RKHS respectively denoted $\mathcal{H}_{k_0}$ and $\mathcal{H}_{k_+}$. More specifically, $\mathcal{H}_{k_+}$ is just $\mathcal{H}_{k_0}$ with the addition of constant functions. The RKHS $\mathcal{H}_{k_+}$ can be used to approximate the integrand $f$ as follows:
\begin{talign*}
\tilde{f}_{\sigma} & \in \argmin \left\{ \|h\|_{\mathcal{H}_+} \text{ s.t. } h \in \mathcal{H}_+, \; h(x_i) = f(x_i), \; i = 1,\dots,m \right\} .
\end{talign*}
Under regularity conditions this provides a unique approximation of the form (see e.g. Proposition 1 in \cite{Briol2019PI}):
\begin{talign*}
\tilde{f}_{\sigma}(x) = k_+(x,X) k_+(X,X)^{-1}f(X),
\end{talign*}
where we have used the matrix notation $[k_+(x,X)]_i = k_+(x,x_i)$, $[f(X)]_i = f(x_i)$ and $[k_+(X,X)]_{i,j} = k_+(x_i,x_j)$ for $i,j \in \{1,\ldots,m\}$. 
The integral of this approximation can be obtained in closed form: 
\begin{talign*}
\Pi[\tilde{f}_{\sigma}] & = \sigma^2 \bm{1}^\top k_+(X,X)^{-1}f(X),
\end{talign*}
where $\bm{1}$ is the vector $(1,\dots,1)^\top$. Finally, the control functional is therefore given by $g_{\sigma}(x) = \tilde{f}_{\sigma}(x) - \Pi[\tilde{f}_{\sigma}]$, which takes the form:
\begin{talign*}
g_{\sigma}(x) & := \left(k_+(x,X) - \sigma^2 \bm{1}^\top\right ) k_+(X,X)^{-1}f(X).
\end{talign*}
To remove the dependence on the regularization due to $\sigma$, we let $\sigma \rightarrow \infty$ and get the \gls{CV} \citep{Oates2017}:
\begin{talign*}
g(x) = k_0(x,X) k_{0}(X,X)^{-1}\left[f(X) - \left(\frac{\bm{1}^\top k_0(X,X)^{-1} f(X)}{\bm{1}^\top k_0(X,X)^{-1} \bm{1}}\right) \bm{1}\right].
\end{talign*}
Properties of control functional estimators have been detailed in \cite{Barp2018,Oates2016CF2,Oates2017}.

\subsection{Control Variates Based on Ensembles of Kernels and Polynomials}\label{app:ensem1}

In our experiments, we also employed a linear combination (or \textit{ensemble}) of \gls{CV}s, one based on a kernel and the other on a polynomial. 
Given a training set $X=\{x_i\}_{i=1}^m$, the \gls{CV} is given by:
\begin{talign*}
g_\theta(x) &= \Delta_x \Phi_{\tilde{\theta}}(x) + \nabla_x \Phi_{\tilde{\theta}}(x) \cdot \nabla_{x}\log \pi(x) +
\bar{\theta} \cdot k_0(x,X)\\
& = \tilde{\theta}^\top b(x) + \bar{\theta} \cdot k_0(x,X),
\end{talign*}
where $\theta = (\tilde{\theta},\bar{\theta})$, $\tilde{\theta}=(\theta_1, \ldots, \theta_p)^\top$, $\bar{\theta}=(\bar{\theta}_1, \ldots, \bar{\theta}_m)^\top$,  $\Phi_{\tilde{\theta}}(x)$ is some polynomial of order $k\in \mathbb{N}$.

This form of \gls{CV} was proposed in \citep{South2020} under the \textit{semi-exact control functionals}, where the authors derived a closed-form expression for the parameter vector $\theta$ under the requirements that: (i) $g_\theta(x_i)=f(x_i)$ for $i=1,\ldots,m$ and (ii) $g_\theta = f$ whenever $f$ belongs to a user-specified finite-dimensional vector space spanned by $b_1,\dots,b_p$.
Requirement (ii) is an \emph{exactness} condition, which motivated the name \textit{semi-exact}. Let
\begin{talign*}
B =\begin{bmatrix}
1 & b_{1}(x_1) & \cdots & b_{p}(x_1)\\
\vdots & \vdots & \vdots & \vdots \\
1 & b_{1}(x_m) & \cdots & b_{p}(x_m)
\end{bmatrix}.
\end{talign*}
Then, under regularity conditions, \cite{South2020} showed that $\tilde{\theta}$ and $\bar{\theta}$
can be found by solving:
\begin{talign*}
\begin{bmatrix}
k_0(X,X) & B \\
B^\top & \bm{0}_{p\times{p}}
\end{bmatrix}
\begin{bmatrix}
\tilde{\theta}  \\
\bar{\theta}
\end{bmatrix} =
\begin{bmatrix}
f(X)  \\
\bm{0}_{p\times{1}}
\end{bmatrix},
\end{talign*}
where $\bm{0}_{p\times{1}}$ is a $p\times1$ column vector of zeros and $\bm{0}_{p\times{p}}$ is
a $p\times{p}$ matrix of zeros.
 For the experiments in this paper we do \emph{not} enforce exactness constraints for mini-batch \gls{SGD} algorithm; as shown in \Cref{fig:logistic_results}, the performance of ensemble \gls{CV}s (a kernel with a polynomial) trained by mini-batch \gls{SGD} and solving linear system is comparable.

\section{Convex Stochastic Optimization} \label{app: convex opt}

The following result establishes convergence over a fixed set $\mathcal{V}_p$ which is linear, both in the idealized scenario where we may directly sample from $\Pi$ and in the practical scenario where we approximate $\Pi$ with MCMC. Let $\sigma_{\text{min}}(M)$ and $\sigma_{\text{max}}(M)$ denote the minimum and maximum singular values of a matrix $M$.

\begin{proposition} \label{thm:convergence_convex_case}
Let $\mathcal{L} : \mathcal{U} \rightarrow L^2(\Gamma)$ be a bounded linear operator for some distribution $\Gamma$ on $\mathbb{R}^d$ and let $\tilde{J}(\theta) := \|f - \theta_0 - \mathcal{L} v_\theta \|_{L^2(\Gamma)}^2$ for $\theta \in \mathbb{R}^{p+1}$. Assume that $\exists u \in \mathcal{U}$ solving the Stein equation $\mathcal{L} u = f - \Pi[f]$. 
Furthermore, assume:
\begin{itemize}[leftmargin=*]
    \item {\normalfont Model:} Any $v \in \mathcal{V}_p$ can be expressed as $v_\theta = \sum_{i=1}^p \theta_i u_i$ where $u_1, \ldots, u_p \in \mathcal{U}$. Furthermore, letting $\psi_0 := 1$ and $\psi_i := \mathcal{L}u_i$, we assume that the $\{\psi_i\}_{i=0}^p$ are linearly independent in $L^2(\Gamma)$.
    \item {\normalfont Optimizer:}  The random variables $x_i^{(t)}$ are distributed according to $\Gamma$, such that $x_i^{(s)}$ and $x_j^{(t)}$ are independent whenever $s \neq t$. Let $\theta^{(t)}$ denote the $t$-th iteration of \gls{SGD}, with stochastic gradient at step $t$ based on batch $\textstyle (x_i^{(t)})_{i=1}^b$. 
    Let $  M_{i,j} := \Gamma[ \psi_i \psi_j ]$ and $ \textstyle \big[M_b^{(t)}\big]_{i,j} := \frac{1}{b} \sum_{k=1}^b \psi_i\big(x_k^{(t)}\big) \psi_j\big(x_k^{(t)}\big)$.
    Suppose the learning rate $(\alpha_t)_{t \in \mathbb{N}}$ satisfies:
\begin{talign*}
&\alpha_t = \textstyle \frac{\beta}{\gamma + t}, \quad \beta > \frac{1}{2 \sigma_{\min}(M)}, \quad \gamma > 0,\\
&\alpha_1 \leq \frac{\sigma_{\min}(M^2)}{2 \sigma_{\max}(M)( \sigma_{\max}(\mathbb{E}[(M_b^{(1)})^2])+\sigma_{\min}(M^2)) } 
\end{talign*}
\end{itemize}
Then, there exists $\nu \geq 0$ such that 
$$
\mathbb{E}[\tilde{J}( \theta^{(t)} )] \leq \textstyle \frac{\nu}{\gamma + t} + \|\mathcal{L}\|_{\mathcal{U} \rightarrow L^2(\Gamma)}^2 \inf_{v \in \mathcal{V}_p} \|u-v\|_{\mathcal{U}}^2.
$$
\end{proposition}
The result is in expectation with respect to the law of $x_i^{(t)}$, and guarantees that the \gls{CV}s trained with \gls{SGD} will converge to the optimal \gls{CV} of the form $g = \mathcal{L}v$, $v \in \mathcal{V}_p$. The second term is an upper bound on $\inf_{v \in \mathcal{V}_p} J(v)$, which will be zero when the assumptions of \Cref{thm:zero_variance} or \Cref{thm:zero_variance_poisson} hold.
The case $\Gamma = \Pi$ corresponds to minimization of $J(v)$ over $v \in \mathcal{V}_p$ using \gls{SGD} with exact sampling from $\Pi$, while the case $\Gamma = \frac{1}{m} \sum_{i=1}^m \delta(x_i)$ corresponds to minimization of the empirical risk $\tilde{J}(\theta) = \tilde{J}_m^{\text{LS}}(\theta)$ using \gls{SGD} with mini-batches drawn from the (fixed) training dataset $(x_i)_{i=1}^m$. For the later case, the theorem is presented for $\tilde{J}_m^{\text{LS}}$, but a similar proof technique could be used for $\tilde{J}_m^{\text{V}}$.

The result does not apply to NNs; indeed, \gls{SGD} is not expected to converge to the global minimum of $\tilde{J}$ when a NN is employed since $\tilde{J}$ will be non-convex. Bounds on the error incurred could however be obtained for alternative algorithms including stochastic
gradient Langevin dynamics \citep{Chau2019,Raginsky2017,Zhang2019nonasymptotic}.

\subsection{Proof of \Cref{thm:convergence_convex_case}}

The following elementary lemma will be required:
\begin{lemma} \label{lem: SPD}
Let $A$ and $B$ be positive semi-definite matrices of equal dimension, such that $\sigma_{\min}(A) \geq \sigma_{\max}(B)$.
Then $A - B$ is also a positive semi-definite matrix.
\end{lemma}
\begin{proof}
Let $A$ and $B$ be $d \times d$ dimensional.
For any $x \in \mathbb{R}^d$ we have that 
\begin{talign*}
x^\top (A-B) x = x^\top A x - x^\top B x & \geq \sigma_{\min}(A) \|x\|_2^2 - \sigma_{\max}(B) \|x\|_2^2 \\
& = (\sigma_{\min}(A) - \sigma_{\max}(B)) \|x\|_2^2  \geq 0.
\end{talign*}
\end{proof}

The implication of our linearity assumption on the model is that the parametrized objective function is a quadratic function in $\theta \in \mathbb{R}^{p+1}$; which simplifies the analysis of \gls{SGD}.

\begin{proof}[\Cref{thm:convergence_convex_case}]
From linearity of $\mathcal{L}$, the objective function that we aim to minimize is
\begin{talign*}
\tilde{J}(\theta) = \left\lVert f - \theta_0 - \sum_{i=1}^p \theta_i \mathcal{L} u_i  \right\rVert_{L^2(\Gamma)}^2 = \left\lVert f - \sum_{i=0}^p \theta_i \psi_i  \right\rVert^2_{L^2(\Gamma)},
\end{talign*}
and we have $\psi_0 = 1$ and $\psi_i = \mathcal{L} u_i$, $i=1,\ldots, p$. 
This can be re-expressed in matrix notation as
\begin{talign}
\tilde{J}(\theta) = \theta^\top M \theta - 2 a^\top \theta + \Gamma[f^2] , \label{eq: tilde J matrix}
\end{talign}
where $M_{i,j} = \Gamma[\psi_i \psi_j]$ and $a_i := \Gamma[f \psi_i]$.
Our two cases of interest are $\Gamma = \Pi$ and $\Gamma = \frac{1}{m} \sum_{i=1}^m \delta(x_i)$ for a fixed set $\{x_i\}_{i=1}^m \subset \mathbb{R}^d$.

Our aim is to verify the preconditions of Theorem 4.7 in \cite{Bottou2018}.
If these are satisfied then we may conclude that, under the assumptions on the learning rate in the statement of  \Cref{thm:convergence_convex_case}, for some constant $\nu \geq 0$,
\begin{talign*}
\mathbb{E}\left[\tilde{J}\left(\theta^{(t)}\right)\right] 
\leq \frac{\nu}{\gamma + t} + \inf_{\theta \in \mathbb{R}^{p+1}} \tilde{J}(\theta).
\end{talign*}
In particular, since we have assumed that $\exists u \in \mathcal{U}$ that solves the Stein equation $\mathcal{L}u = f - \Pi[f]$ and that $\mathcal{L} : \mathcal{U} \rightarrow L^2(\Gamma)$ is a bounded linear operator, the same argument used in the proof of  \Cref{thm:zero_variance} shows that $\tilde{J}(\theta) \leq \|\mathcal{L}\|_{\mathcal{U} \rightarrow L^2(\Gamma)}^2 \|u - v_\theta\|_{\mathcal{U}}^2$, so that 
\begin{talign}\label{eq:final_result_thm1}
\mathbb{E}\left[\tilde{J}\left(\theta^{(t)}\right)\right] 
\leq \frac{\nu}{\gamma + t} + \|\mathcal{L}\|_{\mathcal{U} \rightarrow L^2(\Gamma)}^2 \inf_{v \in \mathcal{V}_p} \|u-v\|_{\mathcal{U}}^2, 
\end{talign}
as claimed. 

Theorem 4.7 in \cite{Bottou2018} requires that $\tilde{J}(\theta)$ is continuously differentiable with $\nabla \tilde{J}$ being Lipschitz. 
This is satisfies in our context, with Lipschitz constant $2\sigma_{\max}(M)$. 
The two remaining conditions that we must verify in order to apply Theorem 4.7 of \cite{Bottou2018} are; (i) the strong convexity property
\begin{talign*}
\tilde{J}(\theta) - \tilde{J}(\vartheta) \geq \langle \nabla \tilde{J}(\vartheta) , \theta - \vartheta \rangle + \frac{l}{2} \|\theta - \vartheta\|_2^2 
\end{talign*}
for some $l > 0$, and (ii) the bound 
\begin{talign}
\mathbb{E}[\|\nabla \tilde{J}_b(\theta)\|_2^2] \leq C_1 + C_2 \|\nabla \tilde{J}(\theta)\|_2^2 \label{eq: Bottou bound}
\end{talign}
for some constants $C_1, C_2$ where $\nabla \tilde{J}_b$ is a stochastic estimate of $\nabla \tilde{J}$ based on $b$ samples from $\Gamma$.
See the discussion of (4.9) in \cite{Bottou2018} for why establishing \eqref{eq: Bottou bound} is a sufficient condition for Theorem 4.7.

First we verify condition (i); that the optimization problem is strongly convex in $\theta \in \mathbb{R}^{p+1}$.
From direct computation with \eqref{eq: tilde J matrix} we obtain that $\tilde{J}(\theta)$ is strongly convex if and only if $(\theta - \vartheta)^\top M (\theta - \vartheta) \geq \frac{c}{2} \|\theta - \vartheta\|_2^2$.
Since $M$ is positive semi-definite and the $\psi_i$ were assumed to be linearly independent in $L^2(\Gamma)$, the matrix $M$ is non-singular and $\sigma_{\min}(M) > 0$.
Thus $\tilde{J}$ is strongly convex with strong convexity constant
$c = 2\sigma_{\min}(M)$. 

It remains only to verify condition (ii).
Let $\psi(x) := (\psi_0,
 \psi_1(x), \dots, \psi_p(x))^\top$.
Recall that, in the $t$th step of \gls{SGD}, the gradient $\nabla \tilde{J}$ is unbiasedly estimated with
\begin{talign*}
\nabla \tilde{J}_b(\theta) & := \nabla \left[ \frac{1}{b} \sum_{i=1}^b \left( f\left(x_i^{(t)}\right) - \psi\left(x_i^{(t)}\right)^\top \theta \right)^2 \right] \\
& = - \frac{2}{b} \sum_{i=1}^b \left( f\left(x_i^{(t)}\right) - \psi\left(x_i^{(t)}\right)^\top \theta \right) \psi\left(x_i^{(t)}\right) ,
\end{talign*}
where the $x_i^{(t)}$ are independent samples from $\Gamma$.
Let $\mathrm{f}$ be the vector with entries $\mathrm{f}_i := f(x_i^{(t)})$, let $a_b$ be the vector with entries $a_{b,j} := \frac{1}{b} \sum_{i=1}^b f(x_i^{(t)}) \psi_j(x_i^{(t)})$, so that $\mathbb{E}[a_b] = a$, and $\Psi_{i,j} := \psi_j(x_i^{(t)})$.
Thus 
\begin{talign*}
    \frac{1}{4} \|\nabla \tilde{J}_b(\theta)\|_2^2 
    & = \frac{1}{b^2} \| (f-\Psi \theta)^\top \Psi \|_2^2 = \frac{1}{b^2} (\mathrm{f} - \Psi \theta)^\top \Psi \Psi^\top (\mathrm{f} - \Psi \theta) \\
    & = \frac{1}{b^2} \left( \theta^\top \Psi^\top \Psi \Psi^\top \Psi \theta - 2 \mathrm{f}^\top\Psi \Psi^\top \Psi \theta + \mathrm{f}^\top \Psi \Psi^\top \mathrm{f} \right) \\
    & = \theta^\top M_b^2 \theta - 2 a_b^\top M_b \theta + a_b^\top a_b
\end{talign*}
where $M_b = \frac{1}{b} \Psi^\top \Psi$ satisfies $\mathbb{E}[M_b] = M$. Similarly,
\begin{talign*}
    \frac{1}{4} \|\nabla \tilde{J}(\theta)\|_2^2 
    & = \theta^\top M^2 \theta - 2 a^\top M \theta + a^\top a.
\end{talign*}
Since \eqref{eq: Bottou bound} is equivalent to non-negativity of
\begin{talign}\label{eq:proof_nonnegative}
& \hspace{-10pt} \frac{1}{4} \left( C_1 + C_2 \|\nabla \tilde{J}(\theta)\|_2^2 - \mathbb{E}[\| \nabla \tilde{J}_b(\theta)\|_2^2] \right) \nonumber \\
& = \theta^\top (C_2 M^2 - \mathbb{E}[M_b^2] )\theta  - 2(C_2 a^\top M - a_b^\top M_b) \theta + \left( \frac{1}{4}C_1 + C_2 a^\top a - a_b^\top a_b \right) ,
\end{talign}
using \Cref{lem: SPD} we choose to set
$C_2 = \sigma_{\max}(\mathbb{E}[M_b^2]) / \sigma_{\min}(M^2)$ to ensure that the matrix $C_2 M^2 - \mathbb{E}[M_b^2]$ is semi-positive definite.
Given this choice of $C_2$, it is possible to take $C_1$ large enough to guarantee that the expression in  \Cref{eq:proof_nonnegative} is is non-negative $\forall \theta \in \mathbb{R}^{p+1}$. 
This verifies (ii).
From Theorem 4.7 in \cite{Bottou2018}, the result follows under the stated assumptions on the learning rate $\alpha_t$.
\end{proof}

\section{Numerical Experiments}
\label{appendix:additional_experiments}

Here we discuss further implementation details for each of the examples in the paper, and also provide additional numerical experiments to complement the results in the main text. 

For all experiments in this paper, the specific
parametric forms of \gls{CV}s that were considered were as follows:
\begin{enumerate}
\item The second order
polynomial class was used for polynomial \gls{CV}s, i.e.,
\begin{talign*}
\Phi_{\theta}(x)=\frac{1}{2}x^\top A x + b^\top x,
\end{talign*}
where $A\in\mathbb{R}^{d\times d}$ is a symmetric matrix, $b\in\mathbb{R}^{d}$, and $\theta = (A, b)$.
\item For the kernel \gls{CV}s, the kernel
$k_0(x, x^\prime)$ in \Cref{eq:kernel0} was used, and we followed \citep{Oates2017} in taking
\begin{talign}\label{eq:base_kernel}
k(x, x^\prime) = (1 + \alpha_1\|x\|_2^2 + \alpha_1\|x^\prime\|_2^2)^{-1}\exp{
-(2\alpha_2^2)^{-1}\|x-x^\prime\|_2^2}
\end{talign}
for hyper-parameters $\alpha_1, \alpha_2 > 0$ to be specified.

\item The NN \gls{CV}s were fully connected layers. For the Gaussian process realization experiment,
the NN had 2 layers and each layer had
50 hidden nodes. 
For other experiments, the NN had 6 layers and each layer had 20 hidden nodes. 
The ReLU activation function was used for all neurons except the output neuron, where the identity function was employed. 
\item The ensemble \gls{CV} of polynomial and kernel was the sum of a degree $2$ polynomial \gls{CV} and a kernel interpolant \gls{CV}. In the case of multiple kernels, the same base kernel was used, but with different choices of hyperparameters. 
\end{enumerate}
The hyper-parameters $\alpha_1$ and $\alpha_2$ in the kernel and ensemble \gls{CV}s were selected via 5-fold cross-validation. In all experiments, the reported computing timings do not include the hyper-parameter tuning time. This
is still fair to compare various \gls{CV}s because all \gls{CV}s and
training methods--\gls{SGD} and exact solution--necessitate tuning hyper-parameters.

The remainder of this section is devoted to reporting details of the experiments that were reported in the main text.
In Section \ref{appendix:polynomial_integrands} we describe the illustrative experiment from the main text and also provide additional experiments, not reported in the main text.
Section \ref{appendix:Genz} contains details for the Genz test function experiment and reports additional results, not contained in the main text.
Section \ref{appendix:GP_realisations} contains details for the Gaussian Process experiment.
In Section \ref{app: pollution} we report an additional experiment that considers posterior inference for a model of atmospheric pollution, not contained in the main text. In Section
\ref{app:multi.kernels} we
present the ensemble \gls{CV} of
two kernels and a polynomial used in the last two experiments of this paper: ordinary differential equations and \emph{sonar} dataset.

\subsection{Numerical Integration of Polynomials}
\label{appendix:polynomial_integrands}

We start by comparing a range of \gls{CV}s on polynomial integrands which are integrated against a Gaussian distribution. This is a good benchmark problem since the integrals can be computed in closed-form, and the performance of each method can hence be studied precisely.

\subsubsection*{Implementation Details}

Consider an integrand which is a sum of $p$ polynomials: 
\begin{talign*} 
f(x) & =  \sum_{j=0}^p \prod_{i=1}^d \alpha_{ji} x_i^{\beta_{ji}}
\end{talign*}
where $x = (x_1,\ldots,x_d) \in \mathbb{R}^d, \alpha \in \mathbb{R}^{p \times d} \; \& \; \beta \in \mathbb{N}^{p \times d}$ (for both matrices, rows correspond to a polynomial, and each column to a dimension of the space). We can easily compute the integral of such a polynomial against a Gaussian distribution $\mathcal{N}(0,\Sigma)$ where $\Sigma = \sigma^2 I_{d \times d}$ using well-known Gaussian identities. In particular, denoting $\pi$ the pdf of this Gaussian distribution, we have:
\begin{talign*}
\Pi[f] & = \int_{\mathbb{R}^d} f(x) \pi(x) \mathrm{d}x = \sum_{j=0}^p \prod_{i=1}^d  \alpha_{ji} \delta_{\{\beta_{ji} \in \{0,2,4,...\}\}} \sigma^{\beta_{ji}} (\beta_{ji}-1)!!
\end{talign*}
where $\delta_{\{\beta_{ji} \in \{0,2,4,...\}\}}$ is an indicator function taking value $1$ when $\beta_{ji}$ is pair and $0$ otherwise. Also, $x!!$ denotes the double factorial (also called semi-factorial), which is the product of all integers from 1 to $x$ that have the same parity (odd or even) as $x$. We can therefore use integration of polynomials against a Gaussian distribution as a test-bed for various MC or \gls{CV} methods.

In this case, we have that $\nabla_{x} \log \pi(x) = -x^2/\sigma^2$, and so $\mathcal{L}_{\text{SL}} u = \Delta u + \nabla u \cdot \nabla \log \pi \rangle$ will itself also be a polynomial whenever $u$ is a polynomial.

\subsubsection*{Additional Experiments}

We start by integrating $f(x)=\sum_{j=1}^{d}(1-x_j)$ against a standard Gaussian: $\mathcal{N}(0,I_{d \times d})$. This integrand is particularly well suited for the polynomial-based \gls{CV}s of \cite{Mira2013,Papamarkou2014} since, in this case, the integrand is itself in the class of functions of the form $\mathcal{L}_{\text{SL}} u_\theta$. We use $n=10^4$ design points obtained by sampling IID from a $\mathcal{N}(0, I_{d \times d})$. A $90/10$ split is used for approximation data and MC data. We compare the polynomial-based \gls{CV}s of degree two trained by solving the least-squares through a linear system, and these same \gls{CV}s trained by \gls{SGD}. The experiments are presented in Fig. \ref{fig:polynomial_integrands} (top row). 

The performance of the \gls{SGD} \gls{CV}s is usually worse initially, but approaches that of the linear system solution as $t$ grows. This holds regardless of the initialization of \gls{SGD} (see the blue lines). Such results are not surprising since the objective function is convex. In low dimensions, the computational cost associated with solving the linear system is low and there is hence not much point to using \gls{SGD}. The main advantage of the \gls{SGD} approach can be observed for large $d$, in which case we obtain a performance close, if not equal, to that of solving the linear system, but at a small fraction of the cost. This advantage of \gls{SGD} also increases with $d$. Note that early stopping of \gls{SGD} could provide good performance with a further reduction of the computational cost.

We also provide additional experiments using kernel-based \gls{CV}s in Fig. \ref{fig:polynomial_integrands_2} (bottom row). Similar conclusions can be made from these plots. Firstly, in low dimension, there is not much point using the \gls{SGD} approach over solving the exact least squares problem, but as the dimensionality of the problem grows, the \gls{SGD} methodology can provide close-to-optimal performance at a fraction of the cost. Secondly, we once again have that early stopping of the \gls{SGD} procedure could provide further significant computational gains.

\begin{figure}[h!]
\centering
\includegraphics[width=\textwidth]{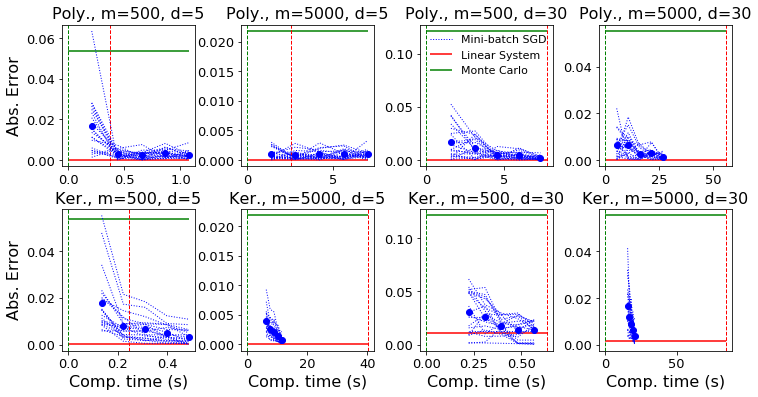}
\caption{\textit{Performance of the polynomial-based \gls{CV}s (top row) and kernel-based \gls{CV}s (bottom row) on polynomial integrands}. We compare \gls{CV}s obtained by solving linear systems with the \gls{CV}s trained using \gls{SGD}. 
The $y$-axis gives the mean absolute error over $20$ different datasets, whilst the dotted line gives the cost of solving the linear system.
Here $20$ realizations (blue dashed lines) are shown and blue dots represent the mean absolute error.
The red lines represent the performance and computational cost of solving the corresponding linear system exactly, our benchmark. 
Similarly, the green lines represent the \gls{MC} estimator with no \gls{CV} used.
}
\label{fig:polynomial_integrands_2}
\end{figure}

\subsection{Genz Test Functions} \label{appendix:Genz}

A popular set of synthetic problems for numerical integration are the Genz test functions introduced in \cite{Genz1984}. These functions, which can all be integrated analytically, were selected to test several difficult scenarios for numerical integration tools based on functional approximation, such as sharp peaks and discontinuities. They are usually defined on $[0,1]^d$, but can easily be transformed to be defined on the whole of $\mathbb{R}^d$, as we discuss next. 

\subsubsection*{Implementation Details}

We consider such a transformation here to keep the setting as close to possible to that of the polynomials. Let $h:\mathcal[0,1]^d \rightarrow \mathbb{R}$ be such a test function. Then, using a change of variables, we get:
\begin{talign*}
\int_{[0,1]^d} h(y) \mathrm{d}y & = \int_{\mathbb{R}^d} h(\Phi(x)) \phi(x) \mathrm{d}x
\end{talign*}
where $\Phi(x)$ is a d-dimensional vector given by $\Phi(x)=(\Phi(x_1),\ldots,\Phi(x_d))$ where $\Phi$ is the cummulative distribution function of a standard Gaussian distribution and $\phi$ is the corresponding probability density function. We therefore have integration problems of the form $\Pi[f] = \int_{\mathbb{R}^d} f(x) \pi(x) \mathrm{d}x$, where $f(x) = h(\Phi(x))$, $\Pi$ is a standard Gaussian, and $h$ is any of the classical Genz functions \citep{Genz1984}, as described in the Table \ref{tab:genz_functions} below. See \url{https://www.sfu.ca/~ssurjano/integration.html} for implementations of these functions in R or MATLAB.

\begin{table}[h!]
\begin{center}
\resizebox{\textwidth}{!}{
\addtolength{\tabcolsep}{3pt}
\begin{tabular}{| c | c | c | } 
\hline 
\textbf{Genz Function} & \textbf{Integrand} & \textbf{Integral} \\ [5pt]
\hline 
& & \\ [-3pt]
 Continuous & $ \displaystyle  \exp\left(-\sum_{i=1}^d a_i |x_i - u_i|\right)$  & $\displaystyle \prod_{i=1}^d a_i^{-1} \left( 2 - \exp(a_i(u_i-1))-\exp(-a_i u_i)\right)$ \\ [15pt]
\hline
 & & \\ [-5pt]
 Corner Peak & 
$ \displaystyle
 \left(1+\sum_{i=1}^d a_i x_i\right)^{-d-1}
$
 & $ \displaystyle \sum_{k=0}^d \sum_{\substack{I \subseteq \{1,\ldots,d\}, \\ |I| = k} } (-1)^{k+d} \left(1+\sum_{i=1}^d a_i - \sum_{j \in I} a_j\right)^{-1} \left(d!\prod_{i=1}^d a_i\right)^{-1}$\\ [15pt]
\hline
&& \\ [-5pt]
Discontinuous &$ \displaystyle  \begin{cases}
0, \text{ if } x_i>u_i \text{ for any } i\\
\exp\left(\sum_{i=1}^d a_i x_i\right), \text{ else}
\end{cases} $ &  $  \displaystyle \prod_{i=1}^d a_i^{-1}\left( \exp(a_i \min(1,u_i))-1 \right) $ \\ [15pt]
\hline
& & \\ [-5pt] 
Gaussian Peak & $ \displaystyle \exp\left(-\sum_{i=1}^d a^2_i (x_i - u_i)^2 \right)$ & $\displaystyle  \left(\frac{\sqrt{\pi}}{2}\right)^d \left(\prod_{i=1}^d a_i^{-1}\right) $ \\
& &  $\displaystyle \times \left(\prod_{i=1}^d \text{Erf}(a_i(1-u_i))-\text{Erf}(-a_i u_i)\right)$\\ [15pt]
\hline
& & \\ [-5pt]
Oscillatory & $ \displaystyle \cos\left( 2\pi u_1 + \sum_{i=1}^d a_i x_i\right)$  & $\displaystyle  \sum_{k=0}^d \sum_{\substack{I \subseteq \{1,\ldots,d\}, \\ |I| = k} } \frac{(-1)^k}{\prod_{i=1}^d a_i}  g\left(2\pi u_1 + \sum_{i=1}^d a_i - \sum_{j \in I} a_j\right)$ \\ [15pt]
& & where $\displaystyle 
g(x) = \begin{cases} 
\sin(x) \text{ if } \text{mod}(d,4) =1\\
- \cos(x) \text{ if } \text{mod}(d,4) =2\\
- \sin(x) \text{ if } \text{mod}(d,4) =3\\
\cos(x) \text{ if } \text{mod}(d,4) =0\\
\end{cases}
$ \\
\hline
& & \\ [-5pt]
Product Peak & $ \displaystyle  \prod_{i=1}^d \left(a_i^{-2} + (x_i - u_i)^2\right)^{-1}$  & $\displaystyle \prod_{i=1}^d a_i\left(\text{arctan}((1-u_i)a_i) - \text{arctan}(-u_i a_i)\right)$ \\ [15pt]
\hline
\end{tabular}
}
\vspace{2mm}
\caption{\textit{Genz Test Functions:} This table contains $6$ test functions defined on $[0,1]^d$, as well as their corresponding integrals against a uniform distribution. The parameter vectors $a = (a_1,\ldots,a_d) \in \mathbb{R}^d_{>0}$ and $u = (u_1,\ldots,u_d) \in [0,1]^d$ can be changed to adapt the difficulty of the integration problem. Their default values are $a = (5,\ldots,5)$ and $u=(0.5,\ldots,0.5)$.}\label{tab:genz_functions}
\end{center}
\end{table}

\subsubsection*{Additional Experiments}

The main numerical results are presented in the main text, but we now highlight additional results. 

Firstly, results (in $d=1$) are provided in Table \ref{table:Genz_functions}. These results focus on kernel-based \gls{CV}s trained with the least-squares objective either by solving the linear system (as per \cite{Oates2017}), or with \gls{SGD} with either $2$ or $5$ epochs. We split the data and assign $50\%$ to constructing the \gls{CV} and $50\%$ for the estimator. In all experiments, the \gls{CV}s provide significant improvement over a MC estimator. Overall, the linear system approach tends to outperform \gls{SGD}, but \gls{SGD} can obtain significant variance reduction at a fraction of the computational cost. Further results are presented in Table \ref{table:Genz_peak_function} in Appendix \ref{appendix:Genz}, which demonstrates that the same conclusion holds for higher-dimensional integrands ($d=5,10,15,20$), and that the $50/50$ split may be suboptimal. Indeed, it is found that assigning a greater proportion of the data on the construction of the \gls{CV} may be preferable, but that this will generally increase computational cost.

\begin{table}[h!]
\begin{center}
\begin{tabular}{|l|l|l|l|l|}
\hline
  \textbf{Integrand} & \textbf{MC} & \textbf{Linear Sys.}  & \textbf{\gls{SGD} 2 Epoc.}  & \textbf{\gls{SGD} 5 Epoc.}  \\ 
  \hline \hline
  Continuous & 2.77e-03& $\bm{3.04}$e-$\bm{04}$& 3.45e-04&  $3.28$e-${04}$ \\
  Corner Peak& 5.76e-03& $\bm{7.07}$e-$\bm{06}$& 1.69e-05& ${9.27}$e-${06}$ \\
  Discontinuous&  2.04e-02& $\bm{2.39}$e-$\bm{03}$& 6.30e-03& ${3.91}$e-${03}$ \\
  Gaussian Peak&   1.47e-03& $\bm{8.84}$e-$\bm{06}$& 1.10e-04& ${1.24}$e-${05}$ \\
  Oscillatory&  4.17e-03& $\bm{3.68}$e-$\bm{06}$& 1.22e-05& ${4.63}$e-${06}$ \\
  Product Peak&  1.37e-03& $\bm{1.79}$e-$\bm{05}$& 1.48e-04& ${2.12}$e-${05}$ \\
  \hline \hline
  Time (sec.) & 7.10e-02 & 5.68e-01 &  1.90e-01 & 4.50e-01  \\
  \hline
 \end{tabular}
 \vspace{2mm}
 \caption{\emph{Performance of kernel \gls{CV}s for the Genz Functions}. We take $n=1000, m=500$. The time presented is an average over repetitions for all six functions (the difference accross integrand was negligeable).}
\label{table:Genz_functions}
\end{center}
\end{table}

Secondly, Table \ref{table:Genz_peak_function} provides additional experiments in the case of the Genz peak function. The table demonstrates that the observation that \gls{SGD} can provide results close to those of LS at a fraction of the price is still true regardless of the dimension.

\begin{table}[h!]
\begin{center}
\begin{tabular}{|l|l|l|l|l|}
\hline
  \textbf{Dim.} & \textbf{MC} & {\textbf{Linear Sys.}} &  {\textbf{\gls{SGD} 2 Epoc.}}   & {\textbf{\gls{SGD} 5 Epoc.}} \\\hline \hline
5&  2.85e-03& $\bm{5.81}$e-$\bm{04}$& 6.51e-04& ${5.84}$e-${04}$ \\
 10 & 2.29e-03& $\bm{1.98}$e-$\bm{04}$& 2.79e-04& ${2.70}$e-${04}$ \\
15 & 2.10e-03& $\bm{4.93}$e-$\bm{04}$& 1.22e-03& ${6.14}$e-${03}$\\
20 & 1.73e-03& $\bm{5.13}$e-$\bm{04}$& 6.35e-04& ${5.90}$e-${04}$\\
\hline \hline
Time (secs.) & 7.00e-02 & 7.60e-01 & 3.30e-01 & 5.95e-01 \\
\hline
\end{tabular}
\vspace{3mm}
\caption{The mean absolute errors and computing times of kernel \gls{CV}s on Genz product peak function
(with parameters
$a=(1.0, \ldots, 1.0)$ and $u=(0.5, \ldots, 0.5)$ of the
same dimension as the integrand)
 of four dimensions: 5, 10, 15 and 20. The total sample size is 1000.}
 \label{table:Genz_peak_function}
 \end{center}
\end{table}

Thirdly, in Table \ref{table:Genz_different_splits}, we provide a comparison of the kernel-based \gls{CV}s for different splits of the data (for solving Stein's equation and MC estimation). We consider four cases: a $50/50$ split (i.e. $50\%$ of the data is used for solving Stein's equation, and $50\%$ for MC estimation), a $70/30$ split, a $90/10$ split and a $100/0$ split. 

The computational cost and mean absolute error (MAE) both depend on the number of data points allocated to each task. The larger we make the proportion of data points allocated to solving Stein's equation, the more expensive the estimator becomes, but this usually comes with an increase in accuracy. Assuming that the number of data points is fixed to $n$. the user is therefore able to chose this split according to the computational power available.

\begin{table}[h!]
\begin{center}
\begin{tabular}{|l|l|l|l|l|}
\hline
{\textbf{Integrand}} & {\textbf{50/50}} & {\textbf{70/30}} &  {\textbf{90/10}} &  {\textbf{100/0}}  \\
\hline \hline
Continuous & $\bm{3.28}$e-$\bm{04}$&  4.82e-04 & 7.40e-04 & ${3.80}$e-${04}$ \\
Corner Peak & $\bm{9.27}$e-$\bm{06}$& 1.57e-05 & 3.26e-05 & ${1.02}$e-${05}$  \\
Discontinuous & 3.91e-03& 7.29e-03 & $\bm{3.2}$e-$\bm{03}$ & ${3.51}$e-${03}$  \\
Gaussian Peak & $\bm{1.24}$e-$\bm{05}$& 2.18e-05 & 2.78e-05 & ${2.05}$e-${04}$ \\
Oscillatory & $\bm{4.63}$e-$\bm{06}$ & 8.47e-06 & 1.67e-05 & 3.25e-05 \\
Product Peak & $\bm{2.12}$e-$\bm{05}$ & ${2.20}$e-${05}$ & 3.03e-05 & 3.63e-05 \\
  \hline \hline
Time (secs.) & 4.70e-01 & 5.79e-01 & 6.89e-01  & 7.40e-01 \\
\hline
\end{tabular}
\vspace{3mm}
\caption{The mean absolute errors of kernel \gls{CV} methods, as a function of the train/test data split. 
The sample size fixed at 1000, and
four train/test splits are used: 50/50, 70/30, 90/10, and 100/0.
The computing times for 5 epochs of mini-batch \gls{SGD} training
are shown in the bottom row.}
\label{table:Genz_different_splits}
\end{center}
\end{table}

\subsection{Integrating Gaussian Processes} \label{appendix:GP_realisations}

We are integrating realizations of a Gaussian process (GP) \citep{Rasmussen2006} with mean function $m(x)=0$ and kernel function
\begin{talign*}
c(x,y;\lambda,\sigma) & =  \lambda^2 \exp\left(-\frac{\|x-y\|_2^2}{2\sigma^2}\right) = \lambda^2 (2 \pi \sigma^2)^{\frac{d}{2}} \phi\left(x|y,\sigma^2,I_{d \times d}\right).
\end{talign*}

\begin{figure}[t]
\centering
\begin{subfigure}{0.35\textwidth}
\includegraphics[width=0.95\linewidth]{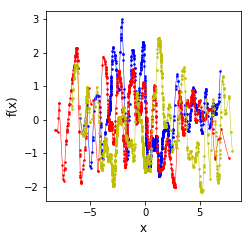} 
\label{fig:subim1}
\end{subfigure}
\begin{subfigure}{0.53\textwidth}
\includegraphics[width=0.95\linewidth]{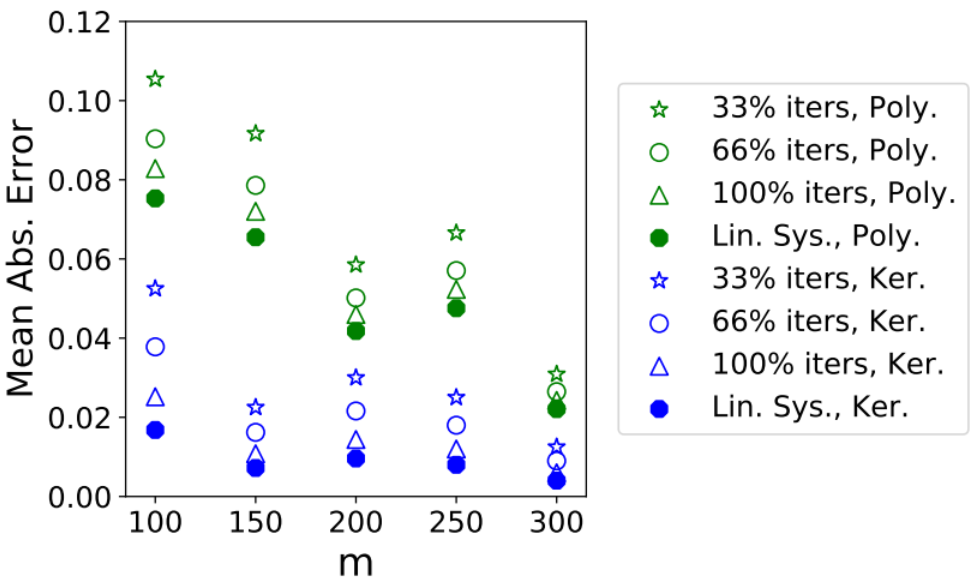}
\label{fig:subim2}
\end{subfigure}
\caption{\emph{Left:} Three realizations from a Gaussian Process.
\emph{Right:} Mean absolute errors
of kernel and polynomial \gls{CV}s evaluated
after $33\%$, $66\%$ and $100\%$ \gls{SGD} training. 
}
\label{fig:appendix_gaussian_process}
\end{figure}
The integral is taken with respect to a mixture of Gaussian distributions with probability density function: $\pi(x) = \sum_{l=1}^L \rho_l \phi(x|\mu_l,\Sigma_l)$, where $\rho = (\rho_1,\ldots,\rho_l) \in [0,1]^d$ is a vector of mixture weights satisfying $\sum_{l=1}^L \rho_l = 1$. For our problems the mean vectors $\mu_1,\ldots,\mu_L$ are generated at random from a zero-mean Gaussian distribution with covariance $3 I_{d \times d}$, and random covariance matrices $\Sigma_1,\ldots,\Sigma_L$ are obtained by taking a matrix $A_l$ with entries uniformly random on $[0,1)$, then setting the covariance $\Sigma_l = A_l^\top A_l$. The mixture weights are also simulated. We simulate the unweighted mixture weights
from a uniform distribution between 0 and 1, and then we normalize them to have
mixture weights.
For this mixture of Gaussians, the score function is given by:
\begin{talign*}
\nabla_x \log \pi(x) =  \frac{\sum_{l=1}^L \rho_l \nabla_x \phi(x|\mu_l,\Sigma_l)}{\sum_{l=1}^L \rho_l \phi(x|\mu_l,\Sigma_l)} = \frac{\sum_{l=1}^L \rho_l \phi(x|\mu_l,\Sigma_l) \Sigma_l^{-1} (x-\mu_l)}{\sum_{l=1}^L \rho_l \phi(x|\mu_l,\Sigma_l)}.
\end{talign*}
We note that the integral of the mean function is $\Pi[m]=0$, and the integrated covariance function is given by:
\begin{talign*}
\Pi[c(x,\cdot)] 
& = 
\lambda^2 (\sqrt{2 \pi} \sigma)^d \sum_{l=1}^L \rho_l \phi\left(x|\mu_l,\Sigma_l +\sigma^2 I_{d \times d}\right).
\end{talign*}
Finally, the integral of the covariance function with respect to the both variables is:
\begin{talign*}
\Pi\Pi[c] 
& =
\lambda^2 (\sqrt{2\pi} \sigma)^d  \sum_{l,m=1}^L \rho_l \rho_m \phi\left(\mu_l \big|\mu_m, \Sigma_l +\Sigma_m+\sigma^2 I_{d \times d}\right).
\end{talign*}
These identities allow us to easily simulate a draw from a Gaussian process and its integral jointly. Indeed, under a Gaussian process model, the vector $(f(x_1),f(x_2),\ldots,f(x_n),\Pi[f])$ is jointly distributed as a multivariate Gaussian distribution with mean $(m(x_1),m(x_2),\ldots,m(x_n),\Pi[m])$ and covariance:
\begin{talign*}
\begin{bmatrix}
    c(x_1,x_1) & c(x_1,x_2)  & \dots  & c(x_1,x_n) & \Pi[c(x_1,x)] \\
    c(x_2,x_1) & c(x_2,x_2)  & \dots  & c(x_2,x_n) & \Pi[c(x_2,x)] \\
    \vdots & \vdots & \ddots & \vdots & \vdots \\
    c(x_n,x_1) & c(x_n,x_2)  & \dots  & c(x_n,x_n) & \Pi[c(x_n,x)] \\
    \Pi[c(x,x_1)] & \Pi[c(x,x_2)] & \dots & \Pi[c(x,x_n)] & \Pi\Pi[c]
\end{bmatrix}.
\end{talign*}
This procedure therefore allows us to create a wide range of examples of varying complexity, by changing the dimension $d$ of the domain, the number $L$ of mixture components and the parameters $\lambda$ and $\sigma$ of the GP covariance function.

In this experiment, all of the samples $x_i$ are included in the training dataset (i.e. $m=n$). We investigate
the performance of two objective functions: $\tilde{J}_{m}^{\text{V}}$
and $\tilde{J}_{m}^{\text{LS}}$. The main
results of this experiment are
displayed in  \Cref{fig:GPrealisations_experiments}.
When implementing \gls{SGD} algorithm, the
batch size is 8 and the number of 
training epochs is 10 for kernel \gls{CV},
and 25 for other \gls{CV}s.

\subsection{Posterior Inference for a Model of Atmospheric Pollutant Detection} \label{app: pollution}

We start with the problem of computing posterior expectations for some Bayesian inference problem linked to the LIDAR (light detection and ranging) experiment, which considers the reflection of light emitted by some laser to detect pollutants in the atmosphere \citep{Ruppert2003}. This dataset consists of $m=221$ observations of the distance travelled before the light is reflected (denoted $\{l_i\}_{j=1}^m$), and of the log-ratios of received light from two laser sources (denoted $\{r_i\}_{j=1}^m$).

Following \cite{Chen2018}, we consider regression with a mean-zero GP model: $r_j = g(l_j) +\epsilon_j$, where $\epsilon_j \sim \mathcal{N}(0,\alpha)$, $\alpha=0.04$. The kernel $c(l,l') = \lambda_1^2 \exp(-\lambda_2^2\|l-l'\|_2^2/2)$ was parameterized with $\phi_1 = \log \lambda_1$ and $\phi_2 = \log \lambda_2$, and a Cauchy prior was placed on $\phi = (\phi_1,\phi_2)$. Denote by $\Pi$ the posterior measure over $\phi$ given the observed data.

We are interested in computing posterior moments $\Pi[\phi_1],\Pi[\phi_2],\Pi[\phi_1^2]$ and $\Pi[\phi_2^2]$, as well as the marginal log-likelihood $\Pi[p(r|l,\phi)]$. The posterior marginal likelihood of the data is defined as the integral of the likelihood $p(y|X,\phi)$ with respect to the posterior on the parameters. This likelihood can be expressed in log-form as below:
\begin{talign*}
\log p(y|X,\phi) & =  -\frac{1}{2}y^\top \left(C_\phi+\alpha I_{n\times n}\right)^{-1} y - \frac{1}{2}\log \left|C_{\phi}+\alpha I_{n \times n}\right| - \frac{n}{2} \log 2 \pi,
\end{talign*}
where $C_{\phi}$ denotes the $n \times n$ matrix with entries $(C_{\phi})_{ij} = c(x_i,x_j)$, where we have made the dependence explicit on the hyper-parameters $\phi$ of the covariance function.

To do so, we use an adaptive MCMC algorithm to sample from $\Pi$. Results are presented in  \Cref{fig:LIDAR_numbersamples} for \gls{CV}s based on polynomials, kernels and NNs all trained with \gls{SGD}. 
When implementing \gls{SGD} algorithm, the
batch size is 8 and the number of 
training epochs is 10 for kernel \gls{CV},
and 25 for other \gls{CV}s.
We notice that the performance can vary significantly based on the Stein space and operator. This clearly demonstrates the advantages of our general methods which can be easily adapted to the problem at hand.

\begin{figure}[h!]
\centering
\includegraphics[width=0.98\textwidth]{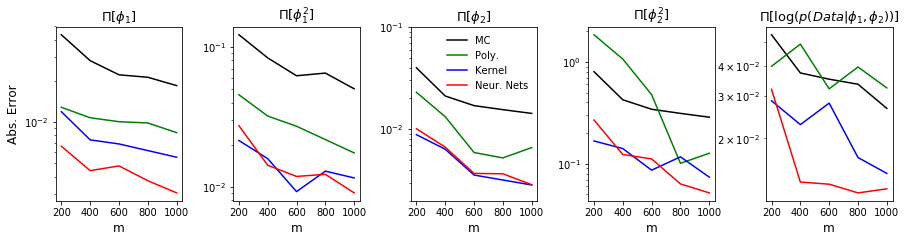}
\caption{\textit{Parameter Inference for Model of Atmospheric Pollutants}. We compute the first two posterior moments of the kernel parameters, as well as the marginal likelihood. }
\label{fig:LIDAR_numbersamples}
\end{figure}

\subsection{Parameter Inference for Ordinary Differential Equations}\label{app:multi.kernels}
In this experiment and the \emph{Sonar Dataset} experiment, we employ an ensemble \gls{CV} of multiple kernels with a polynomial.
Similar to ensemble \gls{CV} of a kernel and a polynomial,
the ensemble \gls{CV} of multiple kernels and a polynomial employs
the sum of a polynomial \gls{CV} and two kernel interpolant.
 We employ two kernels and a polynomial. The two kernels are of the same form as $k_0(x, x^\prime)$ in  \Cref{eq:kernel0}, but they have different hyper-parameters. When
implementing this ensemble \gls{CV}, the kernel in \Cref{eq:base_kernel} is used and it is defined by two hyper-parameters $\alpha_1$ and $\alpha_2$. For two kernels, we chose values of $\alpha_2$ by the median heuristic, setting
one to
\begin{talign*}
\ell = \sqrt{\frac{1}{2}\text{Median}\{\|x_i - x_j\|_2^{2}: 1\leq i<j\leq{m}\}},
\end{talign*}
and the other kernel to $\sqrt{2}\ell$. The other hyper-parameter $\alpha_1$ is the same
for two kernels, and it is tuned via
5-fold cross validation over a grid $\{1.0e$+$6$, $1.0e$+$5$, $1.0e$+$4$, $1.0e$+$3$, $1.0e$+$2$, $1.0e$+$1$, $1.0$, $1.0e$-$1$, $1.0e$-$2\}$. For this ensemble \gls{CV} of multiple kernels and a polynomial, we do not derive the exact solution of parameters, but rather we implement our \gls{SGD} framework to learn this ensemble \gls{CV} on a training dataset.

\subsection{High-dimensional Bayesian Logistic Regression on Sonar Data} \label{appendix:sonar_data}
In this final example, we consider Bayesian logistic regression applied to sonar data from \cite{dheeru2017uci,gorman1988analysis}, as considered in \cite{South2019}.
The parameter is the 61-dimensional regression coefficient $\beta$ and
contains information about the energy frequencies being reflected from either a
metal cylinder ($y = 1$) or a rock ($y = 0$). 
MCMC was used to sample from the posterior of $\beta$ as described in \cite{South2019} and the full output constitutes our ground truth. 
Our task is to approximate the posterior probability that an unlabeled data point $z$ corresponds to a metal cylinder, rather than rock, based on a subset of size $m$ from the MCMC output.
Thus, as for the Madelon data, $f(\beta)=(1+\exp(-z^\top \beta))^{-1}$.

\begin{figure}[ht!]
\centering
\includegraphics[width=.6\textwidth,trim={1mm 0 0mm 0},clip]{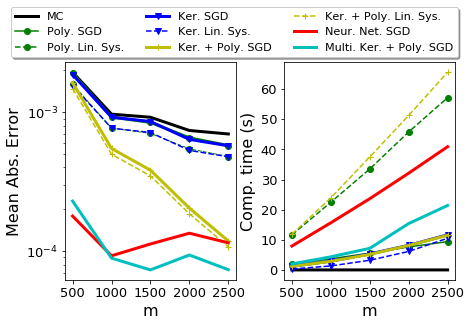}
\caption{\emph{Sonar Dataset}. 
The mean absolute error (left) and compute times (right), as a function of the size $m$ of the training set; based on 20 repetitions. 
}
\label{fig:logistic_results}
\vspace{-3mm}
\end{figure}
We used the same setting as for the Madelon dataset: all MCMC samples were included in the training dataset (i.e. $m=n$), we used batch sizes $b = 8$ over 25 epochs in \gls{SGD} and the loss was $J^{\text{LS}}_m$.  \Cref{fig:logistic_results} compares the performance of different \gls{CV} methods. The two ensemble \gls{CV}s and the NNs perform significantly better than other \gls{CV}s. 
When $m < 1000$, the NNs yield the smallest mean absolute errors, followed by the \gls{CV} with multiple kernels and a polynomial. When $m \geq 1000$, the ensemble
\gls{CV} surpasses NNs. One possible explanation is that for all values of $m$ we used the same multi-layer perceptron (MLP) with
6 layers and 20 nodes in each of them.
Therefore, the NNs size (capacity) remains the same while the training data size $m$ increases. Further growing the depth of NN could lead to an improved performance.
Furthermore, the results for polynomials and kernels demonstrate that our general framework based on \gls{SGD} can achieve comparable MAE with exactly solving the linear systems, but with a fraction of the associated computational overhead.
The compute time of NN in  \Cref{fig:logistic_results} does not capture the time required to manually calibrate \gls{SGD}, so that the ``effective'' compute time is much higher than reported.



\end{document}